\documentclass{article}

\PassOptionsToPackage{numbers,compress}{natbib}
\usepackage[preprint]{neurips_2026}

\usepackage[utf8]{inputenc} 
\usepackage[T2A,T1]{fontenc}    
\usepackage{CJKutf8}       
\newcommand{\kr}[1]{{\upshape\CJKfamily{mj}#1}}
\newcommand{\trad}[1]{{\upshape\CJKfamily{bsmi}#1}}
\newcommand{\rui}[1]{{\fontencoding{T2A}\fontfamily{cmr}\itshape\selectfont #1}}
\usepackage{hyperref}       
\usepackage{url}            
\usepackage{booktabs}       
\usepackage{amsfonts}       
\usepackage{nicefrac}       
\usepackage{microtype}      
\usepackage{xcolor}         
\usepackage{natbib}
\usepackage{graphicx}
\usepackage{float}
\usepackage{amsmath}
\usepackage{amsthm}
\newtheorem{proposition}{Proposition}
\newtheorem{lemma}{Lemma}
\theoremstyle{definition}
\newtheorem{assumption}{Assumption}
\usepackage{enumitem}
\usepackage[table]{xcolor}
\definecolor{best}{RGB}{212,230,249}
\definecolor{second}{RGB}{223,243,246}
\usepackage{wrapfig}

\title{CLORE: Content-Level Optimization for \\ Reasoning Efficiency}

\author{%
  \textbf{Yuyang Wu}\textsuperscript{1}\thanks{Equal contribution.}\hspace{1.2em}
  \textbf{Qiyao Xue}\textsuperscript{3}\footnotemark[1]\hspace{1.2em}
  \textbf{Guanxing Lu}\textsuperscript{1}\footnotemark[1]\hspace{1.2em}
  \textbf{Weichen Liu}\textsuperscript{4} \\
  \textbf{Zihan Wang}\textsuperscript{2}\hspace{1.2em}
  \textbf{Manling Li}\textsuperscript{2}\thanks{Corresponding authors.}\hspace{1.2em}
  \textbf{Olexandr Isayev}\textsuperscript{1}\footnotemark[2] \\[3pt]
  \textsuperscript{1}Carnegie Mellon University \hspace{1em}
  \textsuperscript{2}Northwestern University \\
  \textsuperscript{3}University of North Carolina, Chapel Hill \hspace{1em}
  \textsuperscript{4}University of Pittsburgh
}

\begin{document}
\begin{CJK*}{UTF8}{gbsn}

\maketitle

\begin{abstract}
Reinforcement learning post-training has improved the reasoning ability of large language models, but often produces unnecessarily long, repetitive, or semantically opaque reasoning traces. Existing efficient reasoning methods mainly regulate response length through explicit budgets or length-aware rewards, leaving intermediate reasoning content weakly supervised. We propose CLORE, a content-level optimization framework that improves reasoning efficiency by editing correct on-policy rollouts. CLORE uses an external augmentation model to delete repetitive segments, illegible or task-irrelevant content, and superfluous reasoning after the solution is established, while preserving the final answer. The resulting augmented--original pairs are optimized with an auxiliary reference-free DPO objective alongside standard policy-gradient training. By restricting augmentation to correct trajectories and performing local deletion, CLORE keeps edited rollouts close to the policy distribution and mitigates off-policy mismatch. Experiments on DeepSeek-R1-Distill-Qwen-7B and Qwen2.5-Math-7B across five mathematical reasoning benchmarks show that CLORE improves the accuracy--efficiency trade-off and remains compatible with GRPO, DAPO, Training Efficient, and ThinkPrune. Content-level analyses further show that CLORE reduces repetitive reasoning, illegible content, and post-answer exploration, supporting content-level supervision as a complementary direction to length-level control.
\end{abstract}

\section{Introduction}

Large Language Models (LLMs) have recently achieved remarkable progress on tasks requiring highly structured, multi-step reasoning, including mathematical problem solving~\cite{ahn2024large} and competitive programming~\cite{el2025competitive}. Recent systems such as OpenAI o1~\cite{jaech2024openai} and DeepSeek-R1~\cite{guo2025deepseek} exemplify this trend by explicitly incentivizing models to produce extended Chain-of-Thought (CoT) traces at inference time~\cite{wei2022chain}. These traces commonly incorporate fine-grained intermediate computations, iterative self-verification, and exploration of alternative solution paths~\cite{lightman2023let,yao2023tree,wang2022self}. 

However, emerging evidence suggests that excessively long reasoning trajectories can be counterproductive. Beyond incurring substantial computational overhead, extended chains often exhibit \textit{overthinking}, where additional steps yield diminishing utility and may even degrade accuracy due to the accumulation of uninformative or illegible intermediate reasoning~\cite{kumar2025overthink,li2025thinking,wu2025more}. In particular, we identify three recurring patterns in lengthy reasoning traces, which can be systematically categorized into three types, as illustrated in Figure~\ref{fig:teaser}: \textbf{(1)} repetitive reasoning segments~\cite{jiang2025makes}, \textbf{(2)} uninterpretable or task-irrelevant content~\cite{jose2025reasoning}, and \textbf{(3)} superfluous exploration after the correct solution~\cite{peng2025revisiting,cuesta2025large}.\footnote{Detailed examples are provided in Appendix~\ref{sec:low_quality_reasoning_examples}.} These patterns characterize a form of redundant \emph{uninformative or illegible reasoning}, where additional steps are poorly structured, weakly interpretable, and minimally contributive to problem solving. Such reasoning not only reduces inference efficiency but also introduces noisy supervision signals in RL-based training, leading to unstable optimization dynamics and suboptimal policy updates~\cite{jose2025reasoning,yeo2025demystifying}. Consequently, removing or suppressing these low-quality reasoning components enables more concise and information-dense reasoning trajectories, improving both training and reasoning efficiency.

\begin{figure}
  \centering
  \makebox[\linewidth][c]{%
    \includegraphics[width=\linewidth]{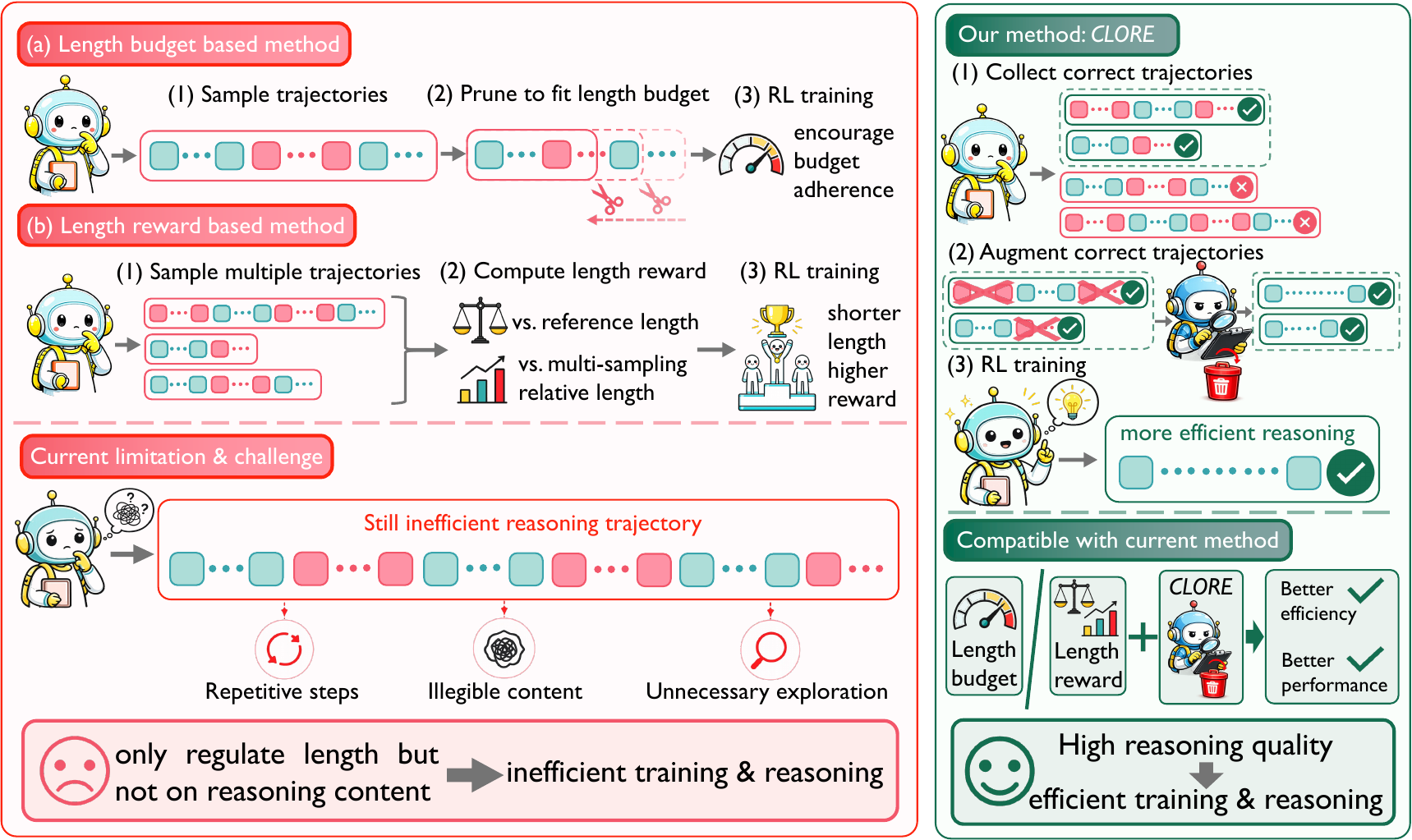}
  }
  \vspace{-0.25in}
  \caption{\textbf{Overview of CLORE.} Existing length-based efficient reasoning methods regulate response length but fail to remove low-quality reasoning content. CLORE improves efficient reasoning by augmenting correct trajectories to delete repetitive, illegible, and uninformative reasoning segments, while remaining compatible with existing efficient RL training methods.}
  \label{fig:teaser}
\vspace{-0.2in}
\end{figure}

Existing efforts to improve reasoning efficiency in LLMs largely focus on regulating the length of generated reasoning. This is typically achieved either by enforcing explicit length budgets~\cite{aggarwal2025l1,han2025token}, where a predefined or adaptive constraint limits the reasoning trajectory, or by designing length-based rewards that encourage shorter outputs through comparisons with reference generations or relative length in trajectory samples~\cite{liu2025learn,yuan2025efficient,dai2025stable,yi2025shorterbetter}. However, these approaches rely on coarse, sequence-level signals and primarily target reasoning length rather than the quality of intermediate reasoning content~\cite{wu2026art,jose2025reasoning}. As a result, they overlook opportunities to directly improve efficiency by supervising reasoning content itself, and lack mechanisms to identify and regularize low-utility, illegible, or non-contributory segments within reasoning traces.

To this end, we propose \textbf{CLORE} (\textbf{C}ontent-\textbf{L}evel \textbf{O}ptimization for \textbf{R}easoning \textbf{E}fficiency), a response-quality-aware method that improves reasoning efficiency by optimizing the content of generated rationales rather than relying solely on sequence-level length control. Our method improves policy-gradient-based reasoning training with LLM-based augmentation of correct trajectories and preference optimization over augmented-original pairs. The framework directly targets low-quality reasoning content, mitigates the off-policy issue introduced by augmented trajectories through Direct Preference Optimization (DPO)~\cite{rafailov2023direct}, and remains fully compatible with both generic policy-gradient algorithms and existing response length level efficient reasoning methods.

In summary, our contributions are as follows:
\begin{itemize}
    \item We propose CLORE, a content-level optimization framework for efficient reasoning. Instead of relying only on sequence-level length control, CLORE augments correct on-policy rollouts by deleting low-quality reasoning content and trains the policy with augmented--original preference pairs through an auxiliary reference-free DPO objective.

    \item We evaluate CLORE on DeepSeek-R1-Distill-Qwen-7B and Qwen2.5-Math-7B across five mathematical reasoning benchmarks, showing consistent accuracy--efficiency gains and compatibility with different length-level efficient reasoning RL training methods, including GRPO, DAPO, Training Efficient, and ThinkPrune.

    \item We provide detailed content-level reasoning analyses and ablation studies to explain the source of the efficiency gains. The analyses show that CLORE reduces repetitive reasoning, illegible or task-irrelevant content, and superfluous post-answer exploration, while the ablations validate the effects of the DPO weight and the robustness on augmentation model.
\end{itemize}

\section{Preliminaries}

Let $x$ denote an input problem sampled from a data distribution $\mathcal{D}$, and let a reasoning policy $\pi_\theta$ generate a trajectory $\tau=(y_1,\ldots,y_T)\sim\pi_\theta(\cdot\mid x)$. The trajectory contains both intermediate reasoning tokens and a final answer. In RL-based reasoning training, the policy is commonly optimized to maximize the expected return $\max_{\theta}\; \mathbb{E}_{x \sim \mathcal{D}} \, \mathbb{E}_{\tau \sim \pi_\theta(\cdot \mid x)} \left[ R(x,\tau) \right]$ using a task reward $R(x,\tau)$, where the reward is primarily determined by final-answer correctness. Although such training can improve reasoning performance, it does not supervise the quality of intermediate reasoning content. As a result, correct trajectories may still contain redundant, illegible, or non-contributory segments. 

Existing efficient reasoning methods mainly regulate the total amount of generated reasoning through length budgets or length-aware rewards. However, reducing length alone does not explicitly distinguish high-quality reasoning from repetitive task-irrelevant or uninterpretable content. We therefore study the problem of \textit{effectively removing uninformative and illegible reasoning content during RL training}, aiming to mitigate overthinking, improve reasoning efficiency, and enhance training efficiency without sacrificing model performance.

\section{Methodology}

\subsection{Problem Formulation}

In this work, we focus on three recurrent forms of low-quality reasoning that frequently appear in otherwise correct long-form trajectories:
\begin{enumerate}[leftmargin=*]
    \vspace{-0.05in}
    \item \textbf{Repetitive reasoning segments}, which redundantly revisit equivalent intermediate states without introducing new problem-solving progress;
    \item \textbf{Illegible or task-irrelevant content}, which weakens semantic coherence and reduces the interpretability of the reasoning process;
    \item \textbf{Superfluous exploration after the correct solution}, which continues the reasoning trajectory despite the final answer already being determined.
    \vspace{-0.05in}
\end{enumerate}
These patterns lengthen responses, reduce interpretability, and introduce noisy training signals. Under policy-gradient optimization, all generated tokens contribute to the update. As a result, non-contributory segments may receive credit alongside useful reasoning, reinforcing inefficient behaviors and destabilizing training process.

A natural approach is to remove low-quality reasoning and train the policy to prefer the augmented trajectory. However, because the augmented trajectory is not sampled from the current policy, it introduces an off-policy distribution mismatch. Our goal is therefore to suppress such content while improving reasoning efficiency and training efficiency in a way that remains compatible with general policy-gradient and length-based efficient reasoning methods.

\begin{figure}[t]
  \centering
  \makebox[\linewidth][c]{%
    \includegraphics[width=\linewidth]{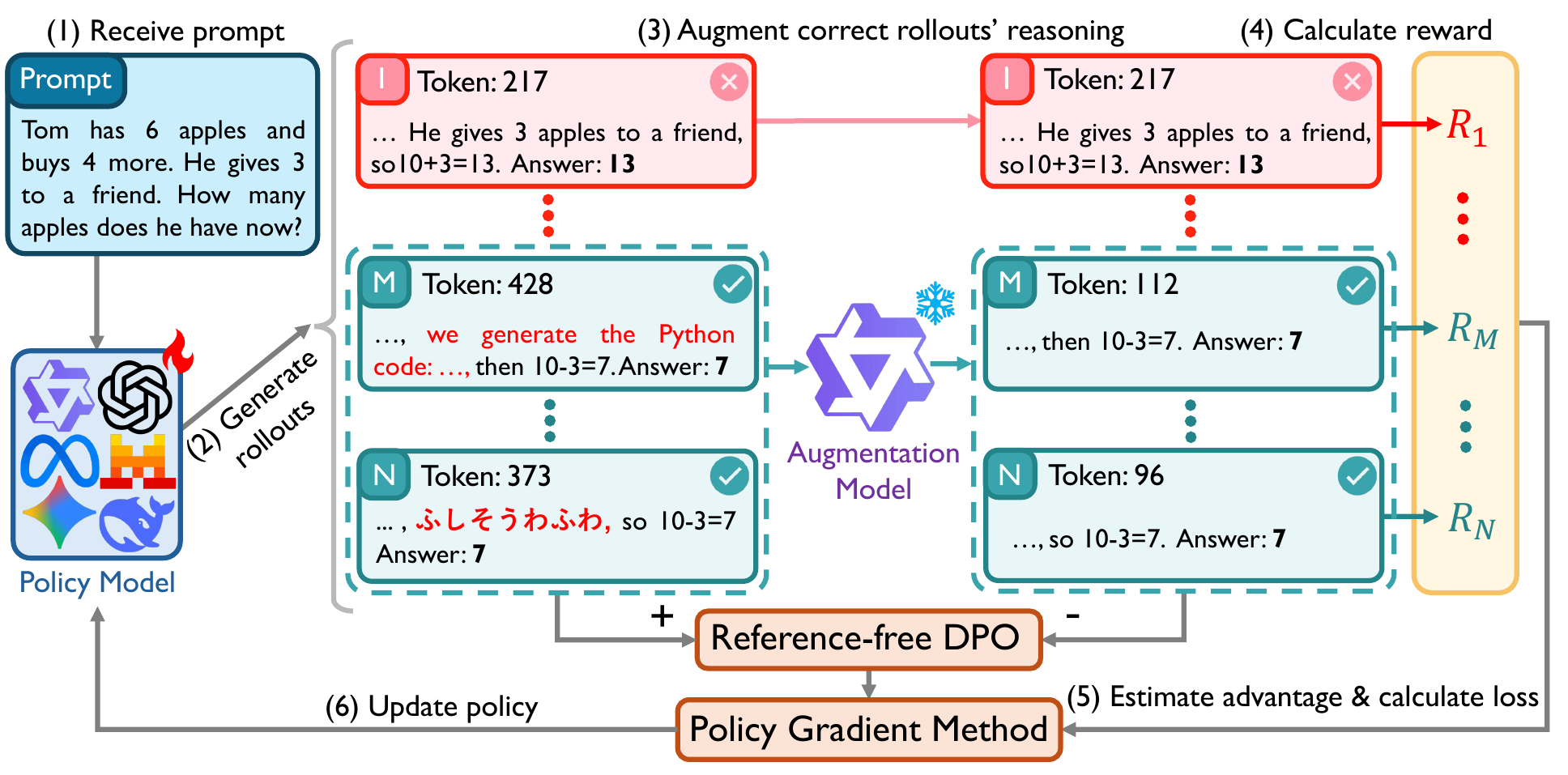}
  }
  \vspace{-0.25in}
  \caption{\textbf{Overview of the CLORE framework.} The policy model samples multiple reasoning trajectories, the correct ones are augmented to remove low-quality content, and the resulting preference pairs are optimized via a reference-free DPO objective jointly with policy-gradient training.}
  \label{fig:overview}
\vspace{-0.1in}
\end{figure}

\subsection{Method}

Our framework augments standard policy-gradient-based RL training with offline-aware supervision on augmented rollouts. As shown in Figure~\ref{fig:overview}, it uses LLM as augmentation model to remove low-quality reasoning content and then leverages the augmented trajectories as training signals to improve reasoning efficiency, while remaining complementary to response length level control.

\textbf{Rollout Generation.} Given a prompt $x$, we sample $N$ trajectories $\{\tau_i\}_{i=1}^{N}$ from the current policy $\pi_\theta(\cdot \mid x)$ and compute their rewards $R(x,\tau_i)$. Let $\mathcal{T}^{\mathrm{cor}}(x)=\{\tau_i \mid R(x,\tau_i)=1\}$ denote the subset of correct trajectories. We restrict augmentation to correct trajectories because they generally contain less low-quality reasoning than incorrect ones and therefore require fewer edits. Consequently, the augmented trajectories remain closer to the original policy samples, which helps mitigate the off-policy mismatch introduced by post-hoc augmentation. 

\textbf{Reasoning Augmentation.}
We apply augmentation only to correct trajectories. For each $\tau \in \mathcal{T}^{\mathrm{cor}}(x)$, an external augmentation model $\mathcal{E}$ removes low-quality reasoning content while preserving the correct solution. The editor is instructed to delete repetitive reasoning spans, illegible or task-irrelevant content, and superfluous reasoning that appears after the solution is already determined. Let $\widetilde{\mathcal{T}}^{\mathrm{cor}}(x)$ denote the union of edited and unedited correct responses, each $\tilde{\tau} \in \widetilde{\mathcal{T}}^{\mathrm{cor}}(x)$ is either an edited trajectory $\mathcal{E}(x,\tau)$ or the original trajectory $\tau$ when no valid modification is produced. For edited trajectories, $\tilde{\tau}$ is expected to be a shorter, more legible, and more information-dense version of $\tau$ while maintaining the same final answer, which induces a natural preference pair $\tilde{\tau} \succ \tau$ on efficient reasoning objective based on reasoning quality rather than answer correctness. To prevent invalid edits from introducing noisy supervision, we retain the pair only if $\tilde{\tau}$ passes the correctness consistency check ${R}(x,\tilde{\tau}) \ge {R}(x,\tau)$. This filtering step ensures that the induced preference reflects improved reasoning quality rather than accidental degradation of solution validity.

\textbf{On-policy Trajectory Preference Learning.} However, since the augmented trajectory \( \tilde{\tau} \) is obtained via post-hoc editing rather than direct sampling from the current policy \( \pi_\theta \), it does not correspond to a standard on-policy rollout for policy-gradient optimization. Instead, it is treated as a derived trajectory that remains locally close to the original on-policy sample \( \tau \) while not strictly following the generative distribution of \( \pi_\theta \). To incorporate supervision from such augmented data, we optimize each \((\tau, \tilde{\tau})\) pair using a reference-free DPO-style objective \( \mathcal{L}_{\mathrm{DPO}}(x,\tau,\tilde{\tau}) = - \log \sigma \!\left( \beta \left[ \log \pi_\theta(\tau\mid x) - \log \pi_\theta(\tilde{\tau}\mid x) \right] \right) \), as introduced in \cite{rafailov2023direct}, and apply it alongside an on-policy optimization regime in which the base trajectories \( \tau \) are sampled from \( \pi_\theta \). In contrast to standard DPO formulations, we do not employ an explicit reference policy. This is because \( \tilde{\tau} \) is constructed by locally deleting low-quality reasoning segments from \( \tau \), which preserves structural and distributional proximity to the original rollout and implicitly constrains deviation from the current policy. This enables stable preference optimization without requiring an additional reference model, while providing a lightweight mechanism for integrating augmentation signal into on-policy training.

\textbf{Compatibility with Policy Gradient Methods.} Our framework is agnostic to the specific policy-gradient backbone. Let $\mathcal{L}_{\mathrm{PG}}$ denote the underlying RL loss used to optimize $\pi_\theta$. In a standard advantage-weighted formulation, it takes the form $\mathcal{L}_{\mathrm{PG}}=-\mathbb{E}_{x\sim\mathcal{D}}\,\mathbb{E}_{\tilde{\tau}\sim\pi_\theta(\cdot\mid x)}
\left[A(x,\tilde{\tau})\log\pi_\theta(\tilde{\tau}\mid x)\right]$, where $A(x,\tilde{\tau})$ denotes the advantage estimation. The final training objective is
$\mathcal{L}=\mathcal{L}_{\mathrm{PG}}+\lambda \, \mathcal{L}_{\mathrm{DPO}}$,  
where $\lambda \ge 0$ controls the strength of the augmentation-induced preference signal, and $\mathcal{L}_{\mathrm{DPO}}$ is averaged over all augmented--original pairs. More generally, $\mathcal{L}_{\mathrm{PG}}$ may be instantiated as PPO-style clipping~\cite{schulman2017proximal,ouyang2022training}, REINFORCE-style updates~\cite{ahmadian2024back,hu2025reinforce++}, GRPO-style group-relative objectives~\cite{shao2024deepseekmath,yu2025dapo}, or other policy-gradient variants. Our method does not alter the backbone optimizer itself, but adds an auxiliary preference-based supervision term on top of it.

\textbf{Compatibility with Length-based Efficient Reasoning Methods.} If the base training objective already includes an additional efficiency-oriented term $\mathcal{L}_{\mathrm{len}}$\footnote{Here, $\mathcal{L}_{\mathrm{len}}$ is a simplified abstraction of length-based method supervision objective.}, induced by a length budget, a length reward, or a sampling-based length regularizer, our method can be incorporated without modification as
$\mathcal{L}=\mathcal{L}_{\mathrm{PG}}+\mathcal{L}_{\mathrm{len}}+\lambda \, \mathcal{L}_{\mathrm{DPO}}$.
This decomposition reflects the complementarity of the two signals. Length-based objectives regulate the overall amount of reasoning at the trajectory length level, whereas our augmentation objective regulates reasoning quality at the trajectory content level. As a result, the proposed framework can strengthen existing efficient reasoning methods by removing low-quality content that response length level control alone cannot reliably identify, while also reducing credit assignment to non-contributory tokens and improving optimization efficiency.

\vspace{-0.1in}
\section{Experimental Setup}
\vspace{-0.1in}

\textbf{Models and Datasets.} We adopt DeepSeek-R1-Distill-Qwen-7B~\cite{guo2025deepseek} and Qwen2.5-Math-7B~\cite{yang2024qwen2} as the base reasoning models, and employ Qwen3-4B-Instruct-2507~\cite{yang2025qwen3} as the augmentation model. For clarity, we denote the resulting model after training as \textsc{CLORE} in the subsequent sections.

We fine-tune the base model on DAPO-Math-17K~\cite{yu2025dapo}, a curated collection of 17K mathematical problems spanning diverse difficulty levels and problem types. The dataset primarily focuses on competition-style mathematics, providing coverage from advanced high-school problems to Olympiad-level challenges. For evaluation, we assess generalization on a suite of challenging mathematical benchmarks. Specifically, we report results on OlympiadBench~\cite{he2024olympiadbench}, Minerva~\cite{lewkowycz2022solving}, and MATH500~\cite{hendrycks2021measuring}, which emphasize complex, multi-step reasoning, as well as competition-oriented datasets including AMC2023 and AIME2025. These benchmarks collectively test out-of-domain mathematical reasoning robustness across varying levels of difficulty and problem styles.

\textbf{Baselines.} To assess the efficiency of our method and its compatibility with existing length-based efficient reasoning methods, we compare the following baselines and apply our method to them:

\begin{itemize}[leftmargin=*]
\vspace{-0.05in}
\item \textbf{GRPO}~\cite{shao2024deepseekmath}: A baseline without any explicit efficient reasoning mechanism.
\item \textbf{ThinkPrune}~\cite{hou2025thinkprune}: A pruning-based efficient reasoning method that enforces a predefined \textit{pruning-based length budget} by explicitly removing reasoning steps that exceed the budget.
\item \textbf{DAPO with Soft Overlong Punishment}~\cite{yu2025dapo}: A scalable RL framework that incorporates \textit{budget-based length rewards} into policy optimization.
\item \textbf{Training Efficient}~\cite{arora2025training}: A \textit{multi-sampling-based length reward} approach that derives length-aware supervision from multiple sampled reasoning trajectories.
\vspace{-0.05in}
\end{itemize}

\textbf{Evaluation.} For evaluation, we generate a single response for each problem instance. For each benchmark, we report three metrics. \textit{Accuracy} is defined as the fraction of correctly answered instances over the total number of evaluated problems. \textit{Output length} measures generation cost, computed as the average number of tokens produced per instance across the benchmark. We also report the \textit{Accuracy-Efficiency (AE) Score}~\cite{luo2025o1}, a composite metric that measures the trade-off between output compression and accuracy retention. A higher AE score denotes better reasoning efficiency, indicating that response length is reduced while preserving correctness more effectively. More detailed implementation details can be found in Appendix~\ref{app:val_setup}

\section{Research Questions and Empirical Analysis}
\subsection{Does CLORE Improve the Accuracy--Efficiency Trade-off?}

\begin{table*}[t]
\centering
\scriptsize
\setlength{\tabcolsep}{3pt}
\renewcommand{\arraystretch}{1.1}
\resizebox{0.97\textwidth}{!}{%
\begin{tabular}{@{}l ccc @{\hspace{8pt}} ccc @{\hspace{8pt}} ccc @{\hspace{8pt}} ccc @{\hspace{8pt}} ccc@{}}
\toprule
& \multicolumn{3}{c}{OlympiadBench}
& \multicolumn{3}{c}{Minerva}
& \multicolumn{3}{c}{MATH500}
& \multicolumn{3}{c}{AMC2023}
& \multicolumn{3}{c}{AIME2025} \\
\cmidrule(lr){2-4}
\cmidrule(lr){5-7}
\cmidrule(lr){8-10}
\cmidrule(lr){11-13}
\cmidrule(lr){14-16}
& Acc. & Len. & AE
& Acc. & Len. & AE
& Acc. & Len. & AE
& Acc. & Len. & AE
& Acc. & Len. & AE\\
\midrule
\multicolumn{16}{@{}l}{DeepSeek-R1-Distill-Qwen-7B} \\
\midrule
Base model (w/o training)
& 38.0 & 5007 & -
& 24.6 & 2958 & -
& 73.0 & 2600 & -
& 75.9 & 2751 & -
& 29.4 & 5740 & - \\
GRPO
& 42.1 & 4200 & 0.32
& \cellcolor{second}29.0 & 2066 & 1.21
& 74.8 & 2138 & 0.32
& 81.9 & 2171 & 0.55
& \cellcolor{second}33.0 & 4831 & 0.57 \\
GRPO\textbf{+CLORE}
& \cellcolor{best}\textbf{43.2} & 3245 & 0.76
& 26.9 & 1629 & 1.38
& \cellcolor{best}\textbf{76.8} & 1514 & 0.92
& 82.7 & 1828 & 0.86
& 32.1 & 4084 & 0.81 \\
DAPO
& \cellcolor{second}42.9 & 3333 & 0.70
& \cellcolor{second}29.0 & 1528 & 1.99
& \cellcolor{best}\textbf{76.8} & 1694 & 0.71
& 82.6 & \cellcolor{second}1808 & \cellcolor{second}0.88
& \cellcolor{best}\textbf{33.7} & 3941 & \cellcolor{second}1.02 \\
DAPO\textbf{+CLORE}
& \cellcolor{second}42.9 & \cellcolor{best}\textbf{2564} & \cellcolor{best}\textbf{1.21}
& 26.5 & \cellcolor{best}\textbf{1255} & \cellcolor{best}\textbf{2.32}
& 74.8 & \cellcolor{best}\textbf{1231} & \cellcolor{best}\textbf{1.30}
& 82.3 & \cellcolor{best}\textbf{1526} & \cellcolor{best}\textbf{1.21}
& 30.0 & \cellcolor{best}\textbf{3256} & \cellcolor{best}\textbf{1.12} \\
Training Efficient
& 42.8 & 3802 & 0.49
& 25.7 & 1862 & 1.17
& 75.2 & 1879 & 0.51
& \cellcolor{best}\textbf{83.8} & 2077 & 0.66
& 31.2 & 4459 & 0.61 \\
Training Efficient\textbf{+CLORE}
& 41.5 & 3079 & 0.78
& 26.5 & 1657 & 1.52
& 74.6 & 1472 & 0.92
& \cellcolor{second}83.5 & 2229 & 0.54
& 32.0 & 4761 & 0.55 \\
ThinkPrune
& 39.7 & 3694 & 0.42
& 26.8 & 1791 & 1.35
& \cellcolor{second}75.8 & 1829 & 0.57
& 80.0 & 1950 & 0.68
& 29.0 & 4353 & 0.54 \\
ThinkPrune\textbf{+CLORE}
& 42.8 & \cellcolor{second}3002 & \cellcolor{second}0.88
& \cellcolor{best}\textbf{29.4} & \cellcolor{second}1408 & \cellcolor{second}2.29
& 74.0 & \cellcolor{second}1394 & \cellcolor{second}1.01
& 82.7 & 1817 & 0.87
& 31.7 & \cellcolor{second}3887 & 0.88 \\

\midrule
\multicolumn{16}{@{}l}{Qwen2.5-Math-7B} \\
\midrule
Base model (w/o training)
& 16.1 & 1116 & -
& 5.5 & 850 & -
& 39.8 & 958 & -
& 54.2 & 571 & -
& 3.3 & 824 & - \\
GRPO
& 32.5 & 1012 & 1.23
& 25.7 & 1199 & 2.31
& 68.0 & 602 & 1.72
& \cellcolor{best}\textbf{76.5} & 897 & -0.10
& \cellcolor{best}\textbf{16.5} & 1104 & 2.73 \\
GRPO\textbf{+CLORE}
& \cellcolor{best}\textbf{38.3} & 631 & 3.21
& \cellcolor{second}27.2 & 542 & 6.76
& \cellcolor{best}\textbf{74.2} & 514 & 2.47
& 73.4 & 661 & 0.17
& \cellcolor{second}15.2 & 637 & \cellcolor{best}\textbf{4.67} \\
DAPO
& 28.5 & 631 & 2.13
& 15.4 & 555 & 3.29
& 59.6 & 500 & 1.87
& 56.1 & 660 & -0.10
& 7.0 & 680 & 1.57 \\
DAPO\textbf{+CLORE}
& 35.7 & 604 & 3.10
& 25.4 & 513 & 6.65
& \cellcolor{second}73.6 & 494 & 2.59
& \cellcolor{second}75.9 & 588 & 0.36
& 9.5 & 640 & 2.71 \\
Training Efficient
& 32.5 & \cellcolor{second}358 & \cellcolor{second}5.29
& 21.7 & \cellcolor{second}314 & \cellcolor{second}9.68
& 64.4 & \cellcolor{second}280 & \cellcolor{second}4.54
& 63.8 & \cellcolor{second}372 & \cellcolor{second}0.81
& 9.4 & \cellcolor{second}411 & \cellcolor{second}3.50 \\
Training Efficient\textbf{+CLORE}
& 35.3 & \cellcolor{best}\textbf{262} & \cellcolor{best}\textbf{8.33}
& 22.1 & \cellcolor{best}\textbf{199} & \cellcolor{best} \textbf{16.16}
& 58.8 & \cellcolor{best}\textbf{225} & \cellcolor{best}\textbf{5.29}
& 64.5 & \cellcolor{best}\textbf{204} & \cellcolor{best}\textbf{2.33}
& 6.2 & \cellcolor{best}\textbf{395} & 2.92 \\
ThinkPrune
& \cellcolor{second}36.4 & 820 & 2.08
& \cellcolor{best}\textbf{27.6} & 667 & 5.39
& \cellcolor{best}\textbf{74.2} & 646 & 1.76
& 72.1 & 767 & -0.01
& 10.0 & 873 & 1.86 \\
ThinkPrune\textbf{+CLORE}
& 34.2 & 560 & 3.23
& \cellcolor{second}27.2 & 569 & 6.39
& 72.4 & 484 & 2.60
& 72.8 & 566 & 0.36
& 9.8 & 556 & 3.40 \\
\bottomrule
\end{tabular}%
}
\vspace{-0.05in}
\caption{\textbf{Evaluation result with DeepSeek-R1-Distill-Qwen-7B and Qwen2.5-Math-7B.} For each model, we report accuracy, length, and AE score for CLORE and its combinations with efficient reasoning methods. The \colorbox{best}{\textbf{best}} and \colorbox{second}{second-best} results are highlighted in each model and metric.}
\label{tab:eval}
\vspace{-0.2in}
\end{table*}

Tables~\ref{tab:eval} show that applying \textsc{CLORE} generally improves AE scores across both DeepSeek-R1-Distill-Qwen-7B and Qwen2.5-Math-7B. The improvement is consistent across different optimization methods, including GRPO, DAPO, Training Efficient, and ThinkPrune. On DeepSeek-R1-Distill-Qwen-7B, applying \textsc{CLORE} improves the average AE score of most baselines by around $0.4$ across the five benchmarks, with the largest gain appearing when \textsc{CLORE} is applied to ThinkPrune on Minerva. On Qwen2.5-Math-7B, the gains are more pronounced, with average AE improvements mostly exceeding $1.0$ for GRPO, DAPO, and ThinkPrune. The largest improvement is observed when applying \textsc{CLORE} to GRPO on Minerva, where the AE score increases by more than $6$ points. These results suggest that \textsc{CLORE} provides a general content-level efficiency signal rather than a model-sensitive or method-specific tuning effect.

Another consistent pattern observed in Table~\ref{tab:eval} is that applying \textsc{CLORE} substantially reduces output length while maintaining competitive accuracy. On DeepSeek-R1-Distill-Qwen-7B, applying \textsc{CLORE} to GRPO, DAPO, and ThinkPrune reduces average response length by roughly $20$--$30\%$ across most benchmarks, while still improving AE on nearly all paired comparisons. On Qwen2.5-Math-7B, the reasoning length reduction effect is even stronger, with length reductions often around $30$--$50\%$ when \textsc{CLORE} is applied to GRPO. Since these length reductions do not generally lead to systematic accuracy collapse, the results indicate that \textsc{CLORE} is not merely imposing a shorter reasoning length. Instead, it improves the density of reasoning by removing repetitive, illegible, or non-contributory segments from reasoning trajectories.

\begin{figure}[t]
  \centering
  \makebox[\linewidth][c]{%
    \includegraphics[width=\linewidth]{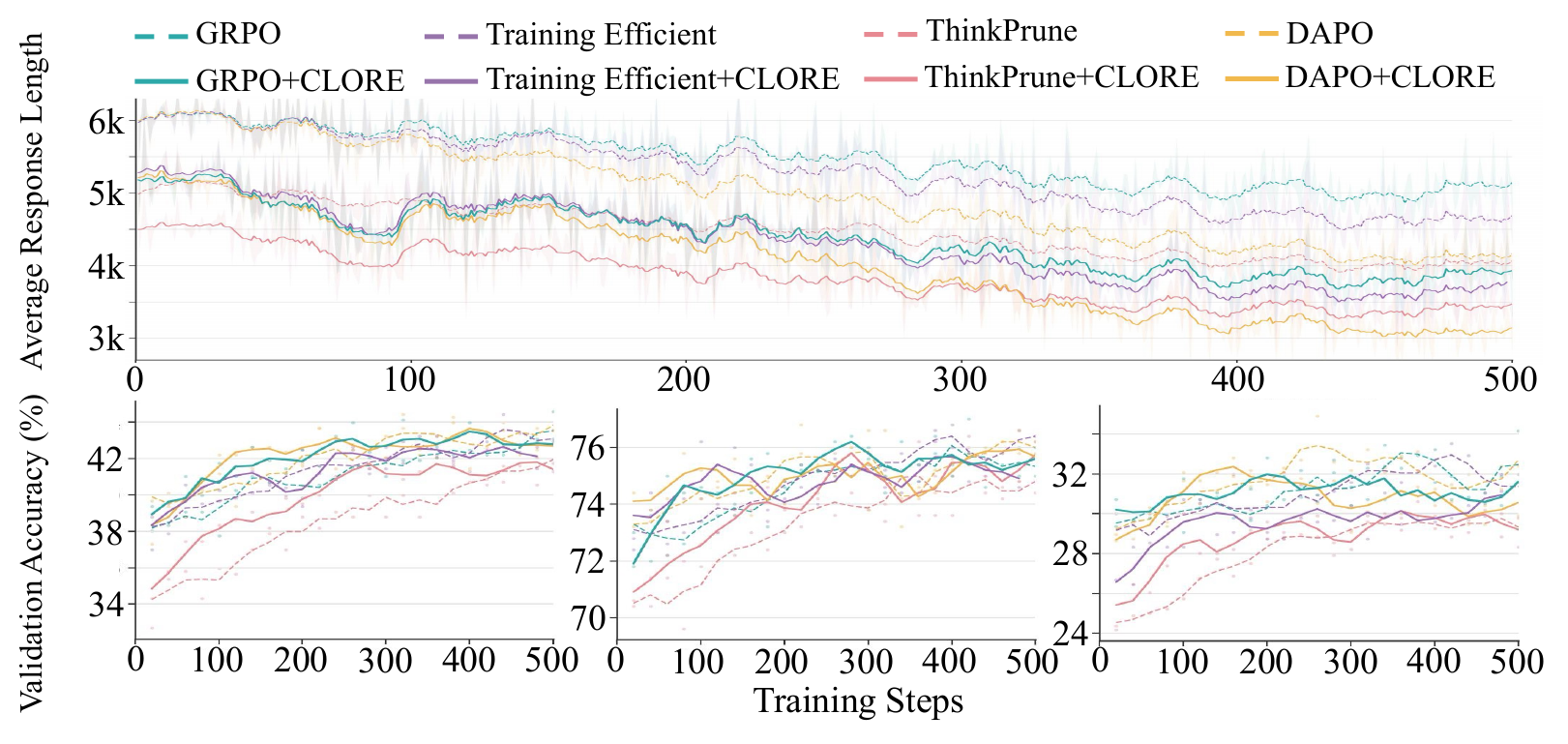}
  }
  \vspace{-0.3in}
  \caption{\textbf{DeepSeek-R1-Distill-Qwen-7B} training dynamics, with average response length on  DAPO-Math-17K and validation accuracy on OlympiadBench, MATH500, AMC2023 (left to right).}
  \label{fig:result_deepseek}
  \vspace{-0.2in}
\end{figure}

\begin{figure}[t]
  \centering
  \makebox[\linewidth][c]{%
    \includegraphics[width=\linewidth]{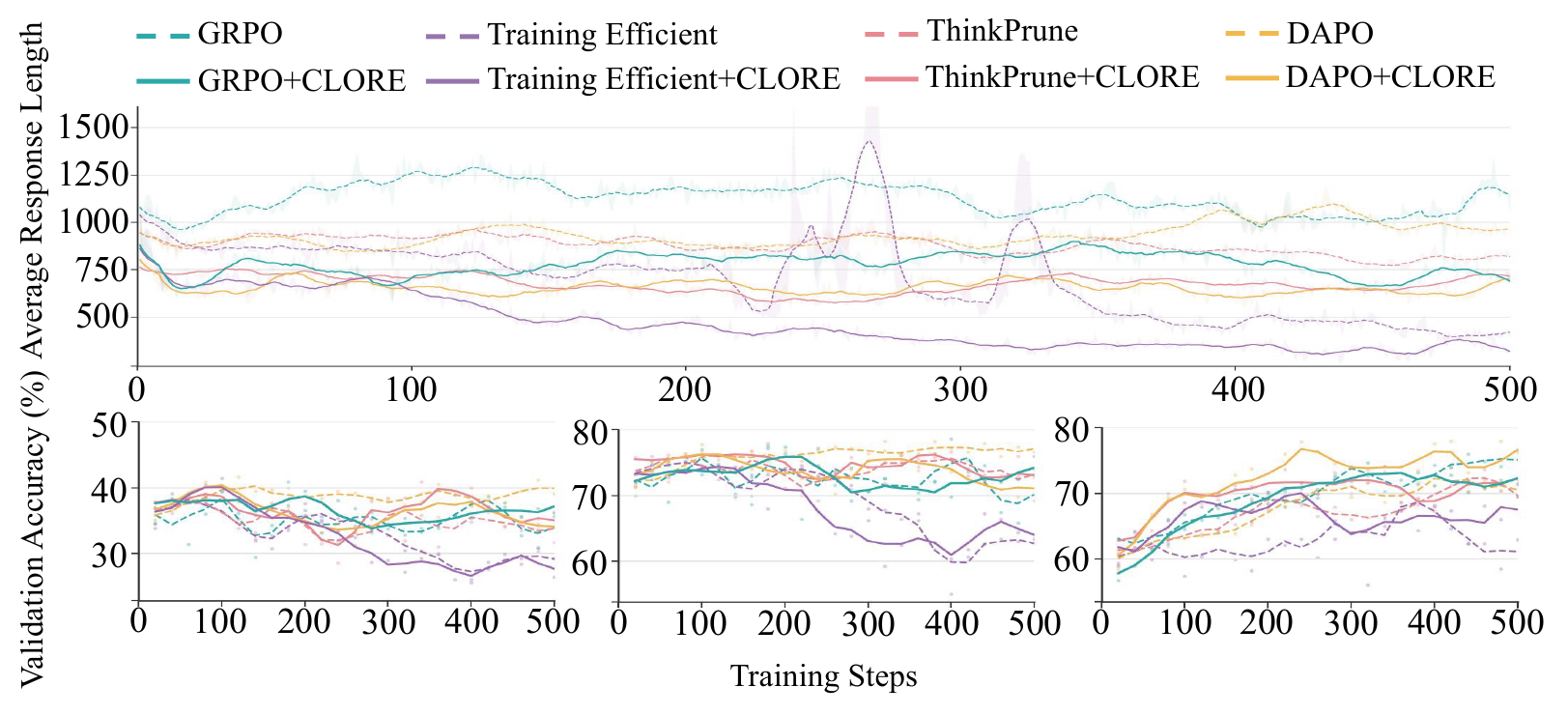}
  }
  \vspace{-0.3in}
  \caption{\textbf{Qwen2.5-Math-7B} training dynamics, with average response length on  DAPO-Math-17K and validation accuracy on OlympiadBench, MATH500 and AMC2023 (left to right).}
  \label{fig:result_qwen}
  \vspace{-0.2in}
\end{figure}

\subsection{Is Content-Level Optimization Complementary to Length-Level Control?}
The evaluation results in Table~\ref{tab:eval} show that efficient reasoning is not equivalent to minimizing response length alone. Although \textsc{CLORE} improves efficiency in most settings, the accuracy--efficiency frontier remains visible. In particular, applying \textsc{CLORE} to Training Efficient on Qwen2.5-Math-7B produces the shortest responses across all benchmarks, but it does not always achieve the highest accuracy. On MATH500, for example, this setting yields much shorter responses but noticeably lower accuracy than applying \textsc{CLORE} to GRPO or DAPO. This suggests that overly aggressive compression can remove reasoning steps that remain useful for correctness. Therefore, the goal of efficient reasoning should not be unconditional length minimization, but the preservation of high-quality reasoning while suppressing redundant or low-quality content.

This observation explains why \textsc{CLORE} is most effective when combined with existing efficient reasoning objectives rather than used as a pure compression mechanism. The strongest AE results are usually obtained when \textsc{CLORE} is applied on top of length-level methods. On DeepSeek-R1-Distill-Qwen-7B, applying \textsc{CLORE} to DAPO yields the best AE score on most benchmarks, showing that content-level supervision can strengthen a length-reward-based RL objective. Applying \textsc{CLORE} to ThinkPrune also improves AE across all five benchmarks, indicating that the proposed method remains useful even when length-budget-based explicit pruning is already used. On Qwen2.5-Math-7B, applying \textsc{CLORE} to Training Efficient produces the shortest outputs and the best AE score on most benchmarks, while applying \textsc{CLORE} to DAPO gives a more balanced accuracy--efficiency trade-off. These trends support the intended complementarity: length-level objectives quantitatively regulate how much reasoning is generated, whereas \textsc{CLORE} qualitatively supervises which reasoning content is informative and worth retaining.

Figures~\ref{fig:result_deepseek} and~\ref{fig:result_qwen} further show that this complementarity emerges during training, not only at final evaluation. Across different training settings, models trained with \textsc{CLORE} tend to maintain shorter response lengths while keeping validation accuracy stable. This indicates that \textsc{CLORE} does not merely post-process outputs after generation. Instead, the augmented--original preference pairs provide a direct training signal that shifts the policy toward generating more concise and information-dense reasoning trajectories. Overall, the results support the central claim that efficient reasoning benefits from complementary optimization of both response length and reasoning content quality.

\vspace{-0.051in}
\subsection{How Do the DPO Weight and Augmentation Model Affect CLORE Performance?}
To answer the question, we use DeepSeek-R1-Distill-Qwen-7B training with GRPO as the base setting, and perform ablation on DPO loss weight and augmentation model.

\begin{figure}[h]
  \centering
  \includegraphics[width=\linewidth]{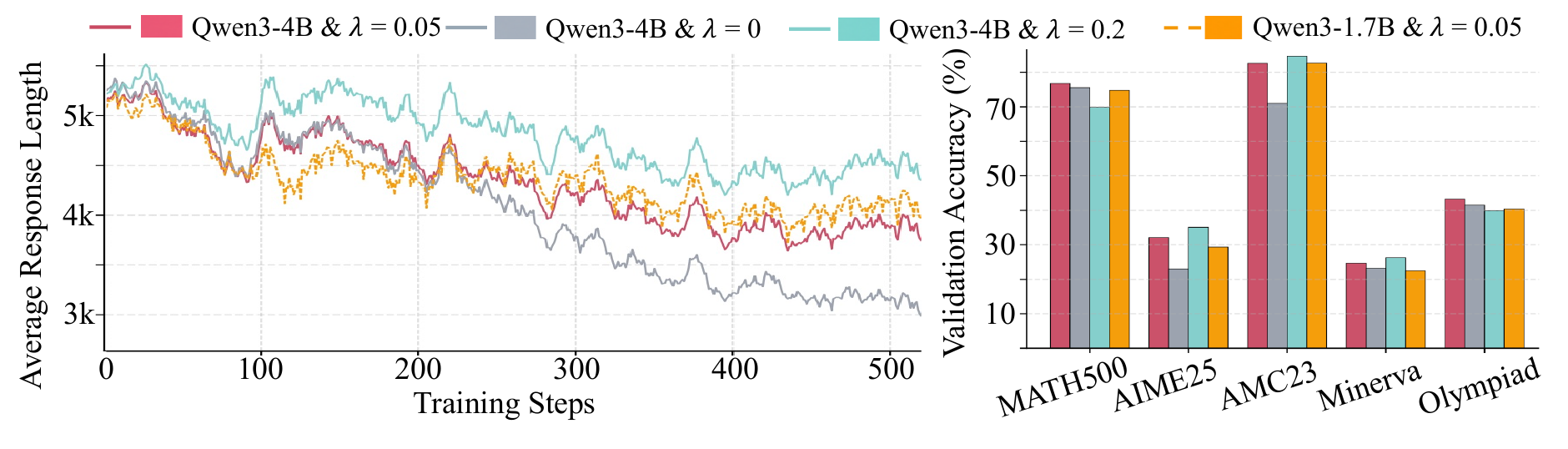}
  \vspace{-0.35in}
  \caption{Ablation study on the DPO weight and augmentation model.}
  \label{fig:ablation}
  \vspace{-0.1in}
\end{figure}

\textbf{Ablation on DPO Weight.}
We ablate the auxiliary preference weight by varying the DPO loss weight $\lambda$. As shown in Figure~\ref{fig:ablation}, removing the DPO term yields the strongest length reduction, but leads to lower validation accuracy on most benchmarks. Increasing the weight to $0.2$ maintains longer responses and does not provide consistent additional accuracy improvement over $0.05$. These results indicate that a moderate DPO weight provides a better accuracy--length trade-off.

\textbf{Ablation on Augmentation Model.}
We further replace the default Qwen3-4B augmentation model with Qwen3-1.7B. As shown in Figure~\ref{fig:ablation}, the smaller augmentation model remains effective, producing similar training dynamics and competitive validation accuracy. Still, Qwen3-4B yields slightly better results, suggesting that CLORE is not critically dependent on a large augmentation model, while a stronger editor can generate cleaner deletion-based preference pairs.

\vspace{-0.1in}
\subsection{Does CLORE Really Suppress Low-Quality Reasoning Content?}
\vspace{-0.05in}
To answer the question, we perform a content-level reasoning analysis using Qwen2.5-Math-7B, where CLORE shows more pronounced accuracy--efficiency improvements, and report metrics averaged across validation benchmarks.

\textbf{Augmentation Analysis.} We analyze the token-level behavior of the augmentation model during
\begin{wrapfigure}{r}{0.34\textwidth}
  \vspace{-0.15in}
  \centering
  \includegraphics[width=0.34\textwidth]{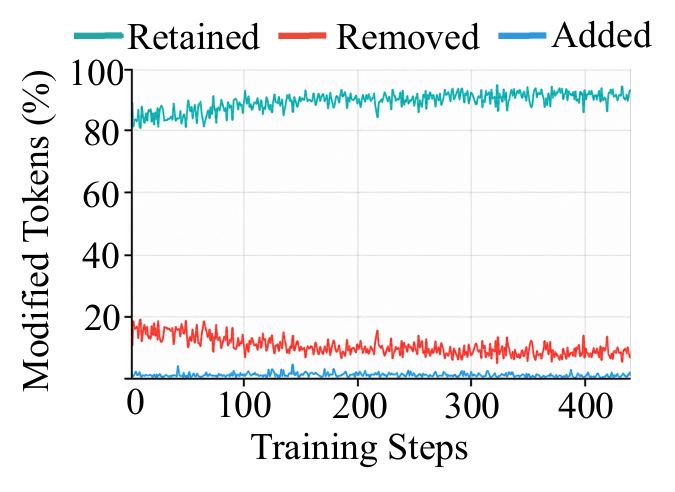}
  \vspace{-0.35in}
  \caption{\small Augmentation analysis.}
  \label{fig:augmentation_analysis}
  \vspace{-0.15in}
\end{wrapfigure}
 training in Figure~\ref{fig:augmentation_analysis}, with examples in Appendix~\ref{app:augmentation_comparison}. Retained tokens dominate the augmented trajectories, removed tokens remain a smaller but persistent fraction, and added tokens are consistently negligible. This indicates that CLORE performs conservative, deletion-based augmentation rather than aggressive rewriting, preserving the structure of correct on-policy 
rollouts while selectively removing low-quality reasoning content. The increasing retention and decreasing removal trends further suggest that augmented rollouts remain close to the current policy distribution, providing a stable content-level supervision.

\textbf{Repetitive Reasoning Analysis.}
As shown in Figure~\ref{fig:repetition_analysis}, we examine whether the proposed content-
\begin{wrapfigure}{r}{0.4\textwidth}
  \vspace{-0.2in}
  \centering
  \includegraphics[width=0.4\textwidth]{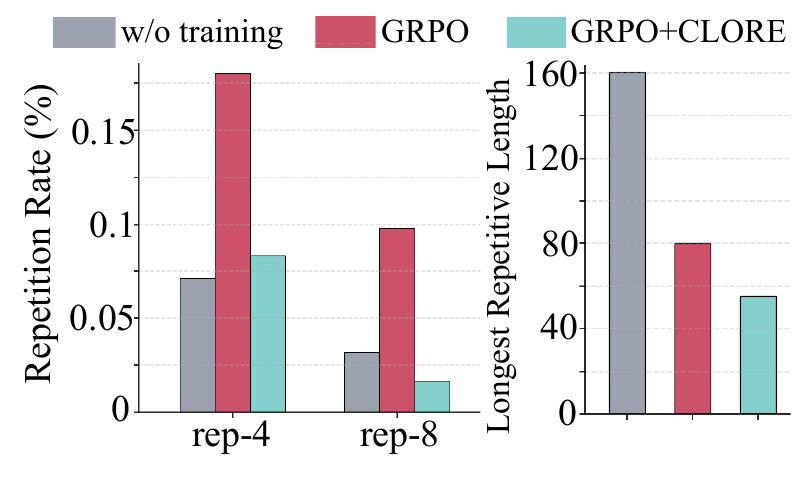}
  \vspace{-0.33in}
  \caption{\small Repetitive reasoning analysis.}
  \label{fig:repetition_analysis}
  \vspace{-0.15in}
\end{wrapfigure}
level supervision reduces repetitive reasoning patterns. We compare the base model, GRPO, and GRPO with CLORE using repetition metrics rep-4, rep-8, and the longest repeated span~\cite{welleck2019neural}. GRPO increases repetition relative to the base model, suggesting that rewards alone may reinforce redundant reasoning during RL training. In contrast, GRPO+CLORE reduces both local $n$-gram repetition and long repeated spans, indicating that CLORE improves reasoning structure by removing repetitive content rather than merely condensing outputs.

\textbf{Post-answer Reasoning Analysis.}
Figure~\ref{fig:post-answer_analysis} evaluates superfluous exploration after the correct 
\begin{wrapfigure}{r}{0.4\textwidth}
\vspace{-0.15in}
\centering
\includegraphics[width=0.4\textwidth]{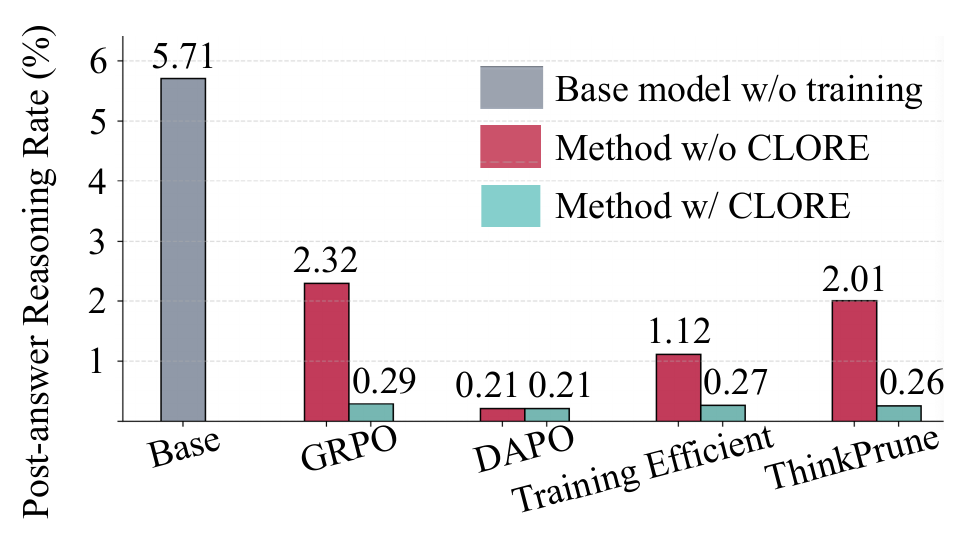}
\vspace{-0.36in}
\caption{\small Post-answer reasoning analysis.}
\label{fig:post-answer_analysis}
\vspace{-0.15in}
\end{wrapfigure}
solution, measured as the fraction of generated reasoning that appears after the final answer has been reached, computed as the average percentage of post-answer tokens in responses. The base model frequently continues reasoning after reaching the solution, while RL training only partially reduces this post-answer content. Adding CLORE consistently suppresses this behavior across training objectives, suggesting that its content-level signal encourages the model to stop once the solution is established.

\textbf{Illegible Reasoning Analysis.}
The results in Figure~\ref{fig:illegible_analysis} provide a complementary view of reasoning 
\begin{wrapfigure}{r}{0.4\textwidth}
\vspace{-0.15in}
\centering
\includegraphics[width=0.4\textwidth]{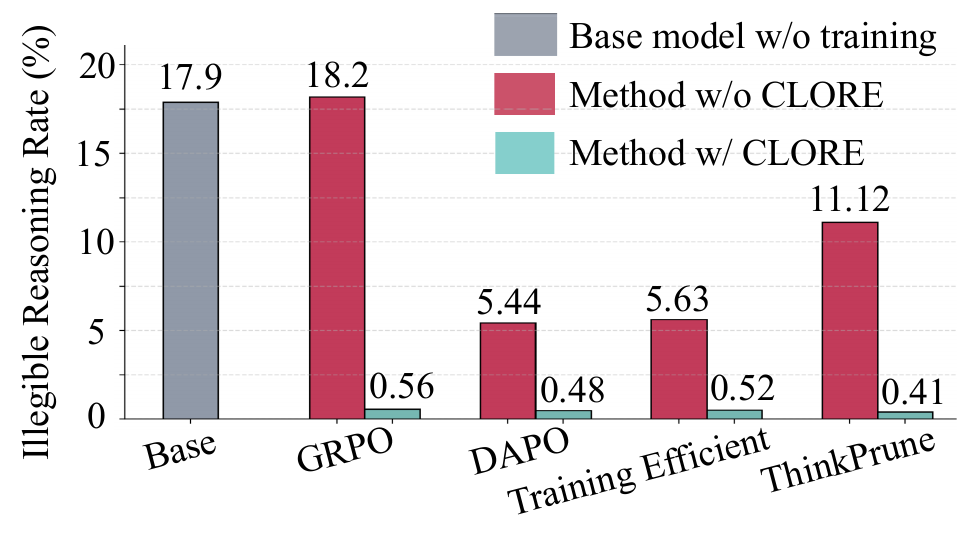}
\vspace{-0.35in}
\caption{\small Illegible reasoning analysis.}
\label{fig:illegible_analysis}
\vspace{-0.15in}
\end{wrapfigure}
quality by measuring the presence of uninterpretable or task-irrelevant content through averaged LLM-as-a-judge result over various strong commercial models\footnote{Detailed settings are provided in Appendix~\ref{app:ill_reasoning}.}. The base model and standard RL-trained methods still produce illegible reasoning, showing that final-answer correctness alone does not sufficiently control intermediate reasoning quality. CLORE consistently reduces such content across training objectives, yielding more interpretable and information-dense reasoning traces.

\section{Related Work}

\textbf{LLM Reasoning.}
Building on the observation that explicit intermediate reasoning improves the performance of large language models, prior work has scaled test-time computation through chain-of-thought prompting, self-consistency, and tree-structured reasoning~\cite{wei2022chain,wang2022self,yao2023tree}. This direction has been further strengthened by RL-based post-training, which encourages longer and more structured reasoning processes~\cite{yeo2025demystifying,guo2025deepseek,reddyprincipled}. However, these training objectives primarily optimize final-answer correctness and do not explicitly constrain the quality of intermediate reasoning steps~\cite{lightman2023let,guo2025deepseek}. As a result, models may generate and reinforce low-utility reasoning content, including redundant, spurious, or uninterpretable steps that do not contribute to the final prediction~\cite{jose2025reasoning,jiang2025makes,peng2025revisiting,cuesta2025large}. Such behavior leads to unnecessarily long reasoning chains, reduced inference efficiency, and noisier RL training signals~\cite{jose2025reasoning,yeo2025demystifying}.

\textbf{Overthinking and Legibility.}
Overthinking has been widely observed as a failure mode of long-reasoning models, where models produce unnecessarily verbose or redundant reasoning traces, often for problems that do not require extensive deliberation~\cite{sui2025stop,chen2025not,zhang2025llms}. It is commonly characterized as computational inefficiency, increasing token usage, latency, and inference cost~\cite{nayab2024concise,su2025between,hassid2025don}. This issue is closely related to chain-of-thought legibility, which concerns whether intermediate reasoning remains understandable and monitorable for humans~\cite{emmons2025pragmatic,kirchner2024prover}. Recent studies show that LLMs can produce illegible reasoning despite correct final answers~\cite{jose2025reasoning,roytburg2026measuring}. Such illegibility often arises from semantically ambiguous, irrelevant, disconnected, or uninterpretable content within the reasoning trace~\cite{emmons2025pragmatic,yang2025llm,wu2024easily}. Since policy-gradient training assigns credit across the generated trajectory, such low-quality content can dilute informative gradients, reduce the signal-to-noise ratio, and destabilize optimization~\cite{zhou2024can,guo2025deepseek,liu2026craft}. 

\textbf{Efficient Reasoning with RL.}
Recent work has explored efficient reasoning by controlling response length during LLM inference and RL training~\cite{aggarwal2025l1,liu2025learn,su2025thinking,li2025aalc}. One line of work imposes predefined or adaptive length budgets and trains models to satisfy these constraints through pruning or budget-aware optimization~\cite{xiang2025just,li2025selfbudgeter,huang2025adactrl}. Another line shapes length-aware rewards using reference lengths, average sampled lengths, or relative comparisons among multiple generated trajectories~\cite{su2025thinking,yuan2025efficient,liu2025learn}. Closely related multi-sampling strategies derive supervision from the distribution of sampled responses, for example by using the shortest response as a baseline~\cite{yi2025shorterbetter,fang2025serl,he2025thinkdial}. Building on this idea, recent work further anchors rewards to the shortest correct reasoning trace, motivated by the observation that an optimal reasoning length may exist~\cite{su2025between,bian2026trims}. These methods regulate length but not intermediate reasoning quality, leaving low-utility, illegible, or non-contributory content weakly penalized and potentially reinforced during RL training.

\vspace{-0.05in}
\section{Conclusion}
\vspace{-0.05in}
We introduced CLORE, a content-level optimization framework for efficient reasoning in RL-trained language models. Rather than only controlling response length, CLORE edits correct on-policy rollouts to remove low-quality reasoning content and trains the policy with augmented--original preference pairs. Experimental results show that CLORE improves the accuracy--efficiency trade-off and complements current length-level efficient reasoning methods. Further analyses demonstrate that CLORE reduces repetitive reasoning, illegible or task-irrelevant content, and unnecessary post-answer exploration. These results suggest that efficient reasoning requires not only length-level control, but also direct supervision over the quality of intermediate reasoning content.

\bibliographystyle{plain}
\bibliography{references}

\medskip


\newpage
\appendix
\newpage

\appendix

\newpage
\section*{\centering Appendix: Table of Contents}
\addcontentsline{toc}{section}{Appendix: Table of Contents}

\vspace{1em}

\noindent
\textbf{A. Limitations}
\dotfill
Page~\pageref{app:limitations}

\vspace{0.8em}

\noindent
\textbf{B. Inefficient and Low-Quality Reasoning Examples}
\dotfill
Page~\pageref{app:lowq_examples}

\vspace{0.8em}

\noindent
\textbf{C. Comparison of CLORE Augmentation}
\dotfill
Page~\pageref{app:augmentation_comparison}

\vspace{0.8em}

\noindent
\textbf{D. Theoretical Analysis}
\dotfill
Page~\pageref{app:theory}

\vspace{0.8em}

\noindent
\textbf{E. Experiment Implementation Details}
\dotfill
Page~\pageref{app:val_setup}

\vspace{0.8em}

\noindent
\textbf{F. Illegible Reasoning Analysis Setting}
\dotfill
Page~\pageref{app:ill_reasoning}

\vspace{0.8em}

\noindent
\textbf{G. Prompt for Augmentation Model}
\dotfill
Page~\pageref{app:momo_prompt}

\vspace{0.8em}

\noindent
\textbf{H. Computation Cost}
\dotfill
Page~\pageref{app:compute_cost}

\vspace{2em}

\section{Limitations.}
\label{app:limitations}
This work focuses primarily on improving reasoning efficiency in mathematical problem-solving tasks, and our experiments are conducted on representative 7B-scale open-weight reasoning models. While these settings provide a controlled and widely used testbed for studying the accuracy-efficiency trade-off, future work could extend CLORE to broader reasoning domains, larger model families, and more diverse task formats such as coding, retrieval-augmented reasoning, and tool-use scenarios. In addition, our analysis emphasizes several measurable forms of low-quality reasoning, including repetition, illegible content, and post-answer exploration; extending the analysis to finer-grained notions of reasoning quality and human interpretability would further deepen the understanding of content-level optimization.

\section{Inefficient and low-quality reasoning examples}
\label{sec:low_quality_reasoning_examples}
\label{app:lowq_examples}

To illustrate the kind of noise CLORE is designed to remove, we sample severs representative rollouts from a CLORE+DAPO training run on Qwen2.5-Math-7B, captured before the augmenting step. Together they cover the four most common pathologies we observe.

\subsection*{Case 1: Verbose code-augmented reasoning}

The model produces the correct answer but routes the entire derivation through a Python program rather than direct symbolic reasoning. The natural-language plan, the code block, and the captured \texttt{output} together inflate the rollout length without adding new content.

\begin{quote}\small
\textsc{Question} (\textit{gt}=318, \textit{acc}=1).\;
The sequence $12, 15, 18, 21, 51, 81, \ldots$ consists of all positive multiples of~$3$ that contain at least one digit that is a~$1$. What is the $50^{\text{th}}$ term of the sequence?

\medskip
\textsc{Rollout (excerpt).}\;
\textit{The problem asks us to find the 50th term in the sequence of positive multiples of 3 that contain at least one digit that is a 1. To solve this, we will write a Python code to generate the multiples of 3, filter out those that contain the digit 1, and then find the 50th term in this filtered sequence.}
\begin{verbatim}
```python
def contains_digit_one(n): return '1' in str(n)
def find_50th_valid_multiple():
    valid_multiples = []; multiple_of_3 = 3
    while len(valid_multiples) < 50:
        if contains_digit_one(multiple_of_3):
            valid_multiples.append(multiple_of_3)
        multiple_of_3 += 3
    return valid_multiples[49]
print(find_50th_valid_multiple())
216
\end{verbatim}
\textit{The 50th term \dots is $\boxed{216}$.}
\end{quote}

\subsection*{Case 2: Code verification that contradicts a wrong analytical answer}

The model derives an analytical answer ($25$), runs a one-line ``Python verification'' that simply prints the same incorrect number, and finally writes a different and also incorrect boxed answer ($2500$). The code block here is pure noise: it neither verifies the analytical claim nor catches the error.

\begin{quote}\small
\textsc{Question} (\textit{gt}=650, \textit{acc}=0).\;
Let $V$ be the set of vertices of a regular $25$-sided polygon with center at point $C$. How many triangles have vertices in $V$ and contain $C$ in their interior?

\medskip
\textsc{Rollout (excerpt).}\;
\textit{$\ldots$ Therefore, for each choice of the first vertex, there is exactly one valid choice for the third vertex. \dots This gives us $25$ ways to choose the vertices $(a,b,c)$.}
\begin{verbatim}
number_of_sides = 25
number_of_triangles = number_of_sides
print(number_of_triangles)
25
\end{verbatim}
\textit{The number of triangles \dots is $\boxed{2500}$.}
\end{quote}

\subsection*{Case 3: Multilingual gibberish and collapsed reasoning}

Occasionally the model emits a short sentence mixing tokens from several languages with no actual reasoning. These rollouts are short but contribute strongly to the degeneration patterns measured by rep-$n$ and the longest-repeated-substring metric.

\begin{quote}\small
\textsc{Question} (\textit{gt}=7, \textit{acc}=0).\;
Find the sum of the non-repeated roots of the polynomial $P(x) = x^{6} - 5x^{5} - 4x^{4} - 5x^{3} + 8x^{2} + 7x + 7$.

\medskip
\textsc{Rollout (verbatim).}\;
\textit{The Poole \kr{에게} \rui{являются} {\upshape 三\trad{個}} choices for the final answer of the |deg(h(x))| to be no greater than~$8$, as said by, {\upshape 综艺渺逃避}.}
\end{quote}

\subsection*{Case 4: Long but incorrect chain-of-thought}

The model writes out a multi-step natural-language analysis, attaches a Python brute-force simulation, and confidently picks the wrong answer despite the simulation enumerating all permutations. These are the longest rollouts in the batch and the ones the CE filter most often retains as candidates for deletion-only cleanup.

\begin{quote}\small
\textsc{Question} (\textit{gt}=127, \textit{acc}=0).\;
There are seven cards in a hat, and card $k$ shows the number $2^{k-1}$ for $k=1,\ldots,7$. Solarin picks cards at random one at a time until the sum exceeds $124$. What is the most probable sum he can get?

\medskip
\textsc{Rollout (excerpt).}\;
\textit{$\ldots$ The possible sums he can end up with are $125$, $126$, $127$. \dots The probability of ending with a sum of $127$ is the probability of picking the card with value $127$ first, which is $\tfrac{1}{7}$. \dots From the above analysis, the most probable sum is $126$, as it has the highest probability of occurring.}
\begin{verbatim}
import itertools
cards = [1, 2, 4, 8, 16, 32, 64]
all_sequences = list(itertools.permutations(cards))
def find_sum(sequence):
    current_sum = 0
    for card in sequence:
        current_sum += card
        if current_sum > 124: return current_sum
sum_counts = {}
for sequence in all_sequences:
    end_sum = find_sum(sequence)
    sum_counts[end_sum] = sum_counts.get(end_sum, 0) + 1
most_probable_sum = max(sum_counts, key=sum_counts.get)
\end{verbatim}
\textit{$\ldots$ The most probable sum is $\boxed{126}$.}
\end{quote}

\section{Comparison of CLORE augmentation}
\label{app:augmentation_comparison}

Since CLORE only cleans rollouts whose final answer is already correct, we now show three matched pairs of the same correct rollout before and after cleaning, taken from training step~21 of the Qwen2.5-Math-7B + DAPO + CLORE run. The boxed answer is identical in each pair; the augmented version removes content that is mechanical, redundant, or actively contradictory.

\paragraph{Case 1: Removing a redundant code-verification block.}
After deriving $r = 62/63$ analytically and arriving at $m+n = \boxed{125}$, the rollout appends a SymPy verification block whose only purpose is to recompute the same answer. CLORE deletes the entire verification while preserving the analytical derivation and the boxed answer.

\textsc{Question.}\;
\textit{An infinite geometric series has sum 2000. A new series, obtained by squaring each term, has sum 16 times the sum of the original. The common ratio is $m/n$ in lowest terms; find $m+n$.}

\medskip
\textsc{Before CLORE} (3{,}001 characters; tail of the response):
\begin{quote}\small
[\dots derivation arriving at $r = \tfrac{124}{126} = \tfrac{62}{63}$ and $m+n=125$ \dots]\\[2pt]
Let's confirm this with Python and sympy:
\begin{verbatim}
import sympy as sp
a, r = sp.symbols('a r')
sum_original = sp.Eq(a / (1 - r), 2000)
a_expr = sp.solve(sum_original, a)[0]
sum_new_series = sp.Eq(a**2 / (1 - r**2), 32000)
sum_new_series_sub = sum_new_series.subs(a, a_expr)
r_value = sp.solve(sum_new_series_sub, r)
r_value = [sol.evalf() for sol in r_value if sol.is_real and sol > 0][0]
common_ratio = sp.Rational(r_value).limit_denominator()
m, n = common_ratio.as_numer_denom()
print(m + n)
\end{verbatim}
\verb|```output 125 ```|\\
The sum of the numerators and denominators of $\tfrac{62}{63}$ is $\boxed{125}$.
\end{quote}

\textsc{After CLORE} (2{,}221 characters; $-780$ chars, $-26\%$):
\begin{quote}\small
[\dots same derivation arriving at $r = \tfrac{62}{63}$ \dots]\\[2pt]
The common ratio $r$ is $\tfrac{62}{63}$, where $62$ and $63$ are relatively prime. So $m=62$ and $n=63$, and $m + n = \boxed{125}$.
\end{quote}

\paragraph{Case 2: Removing a code block that contradicts the analytical solution.}
The model arrives at the correct answer $6$ analytically, then runs a Python verification that returns \texttt{False} due to floating-point arithmetic, and spends another 2{,}000 characters trying to debug the verification before re-deriving the same answer by hand. CLORE removes the entire failed-verification chain.

\textsc{Question.}\;
\textit{There exist positive integers $A,B,C$ with no common factor $>1$ such that $A\log_{200}5 + B\log_{200}2 = C$. Find $A+B+C$.}

\medskip
\textsc{Before CLORE} (3{,}566 characters; tail of the response):
\begin{quote}\small
[\dots correct derivation: $A=2C$, $B=3C$, so with $\gcd(A,B,C)=1$ we get $A=2$, $B=3$, $C=1$ and $A+B+C = 6$ \dots]\\[2pt]
Let's verify this with Python code to ensure the solution is correct.
\begin{verbatim}
import math
A, B, C = 2, 3, 1
left_side = A * math.log(5, 200) + B * math.log(2, 200)
right_side = C
print((left_side == right_side, A + B + C))
\end{verbatim}
\verb|```output (False, 6) ```|\\
It appears there was an error in the verification. [\dots tries again with \texttt{math.log(5)/math.log(200)}, gets \texttt{(False, 6)} again \dots]
It seems there is still a discrepancy. Let's manually check\dots\ [re-derives the same analytical solution by hand] \dots\ Thus, $A=2$, $B=3$, $C=1$ are correct, and $A+B+C = \boxed{6}$.
\end{quote}

\textsc{After CLORE} (1{,}348 characters; $-2{,}218$ chars, $-62\%$):
\begin{quote}\small
[\dots augment analytical derivation \dots]\\[2pt]
So the values of $A,B,C$ are $2,3,1$. Therefore $A+B+C = 2+3+1 = \boxed{6}$.
\end{quote}

\paragraph{Case 3: Removing multilingual and corrupted code.}
The model arrives at a partially-formed solution, then emits a long span of syntactically broken Python interleaved with unrelated tokens (Chinese fragments, proper names, \texttt{\_\_TIMEOUT\_\_}, stray brackets) before recovering with a clean Python block that produces the correct answer. CLORE removes the corrupted span and keeps only the working code.

\textsc{Question.}\;
\textit{Let $N = 9 + 99 + 999 + \cdots + \underbrace{99\ldots 9}_{321\text{ digits}}$. Find the sum of the digits of $N$.}

\medskip
\textsc{Before CLORE} (2{,}192 characters; middle of the response):
\begin{quote}\small
[\dots set up $N_k = 10^k - 1$ \dots] After calculating and taking the sum of the task, follow with \textbf{sum of the digits.}
\begin{verbatim}
# Python code to perform the sum of the series
def sum_of_series_mat[n_sum: int] ) sum(10**for in range(n_sumrov)
 + reindex\(10**幼svg](n_sum))
 return sumo\ Pablo Werner l a strateg.
**[clarification link](problem/5 Petr\404hnske)\ 历史上最 denim of ABOVEAi clear?
Our piteous兰坷水泵可乐\
宣布.是 complimentary -> eu
... There is an error in the J(\ J\ T\SGPromise - \component rect Step}
萧越 - \ To the funds * of within\ of of chefs o
... ___TIMEOUT___套
 Donovan\ hoot ... Beaver
\end{verbatim}
\begin{verbatim}
# sum the series from one to 321 digits of 9
def calculate_sum_of_series():
    total_sum = 0
    for k in range(1, 322):
        total_sum += 10**k - 1
    return total_sum
N = calculate_sum_of_series()
print(sum(int(d) for d in str(N)))
\end{verbatim}
\verb|```output 342 ```|\\
The sum of the digits of $N$ is $\boxed{342}$.
\end{quote}

\textsc{After CLORE} (1{,}126 characters; $-1{,}066$ chars, $-49\%$):
\begin{quote}\small
[\dots same setup \dots] After calculating and taking the sum of the task, follow with \textbf{sum of the digits.}
\begin{verbatim}
def calculate_sum_of_series():
    total_sum = 0
    for k in range(1, 322):
        total_sum += 10**k - 1
    return total_sum
N = calculate_sum_of_series()
print(sum(int(d) for d in str(N)))
\end{verbatim}
\verb|```output 342 ```|\\
The sum of the digits of $N$ is $\boxed{342}$.
\end{quote}

\paragraph{Summary of cleaning effect.}
Across these three cases, CLORE removes 780, 2{,}218, and 1{,}066 characters respectively, with mean reduction $\approx 46\%$ without changing the final answer. 

\section{Theoretical Analysis}
\label{app:theory}

This section gives a brief theoretical perspective on why CLORE's combination of a policy-gradient (PG) update on the cleaned trajectory $\tilde\tau$ and a reference-free DPO regularizer over deletion-augmented preference pairs yields stable content-level supervision. The argument is intentionally elementary, and we rewrite the joint gradient as a weighted MLE in which PG carries the main reinforcement signal that raises $\pi_\theta(\tilde\tau\mid x)$ while DPO contributes a bounded regularization term that re-engages $\pi_\theta(\tau\mid x)$ through the contrastive log-ratio, then use this form to read off two properties that justify CLORE's design.

\paragraph{Setup and notation.}
For prompt $x \sim \mathcal{D}$ and policy $\pi_\theta$, let $\tau \sim \pi_\theta(\cdot \mid x)$ be an on-policy rollout with reward $R(x,\tau)\in\{0,1\}$. For each correct rollout, the augmenter $\mathcal{E}$ produces $\tilde\tau = \mathcal{E}(x,\tau)$ with advantage $A(x,\tilde\tau)$. Writing $\Delta := \beta\!\left[\log\pi_\theta(\tau\mid x)-\log\pi_\theta(\tilde\tau\mid x)\right]$ for the implied (reference-free) preference logit, the CLORE objective on a single $(x,\tau,\tilde\tau)$ tuple is
\[
\mathcal{L}(x,\tau,\tilde\tau)
\;=\;
\underbrace{-A(x,\tilde\tau)\log\pi_\theta(\tilde\tau\mid x)}_{\mathcal{L}_{\mathrm{PG}}}
\;+\;
\lambda\underbrace{\big[-\log\sigma(\Delta)\big]}_{\mathcal{L}_{\mathrm{DPO}}},
\]
with $\sigma(z)=1/(1+e^{-z})$, preference strength $\beta>0$, and DPO weight $\lambda\ge 0$. We write $T=|\tau|$ for the rollout length and $h$ for an arbitrary prefix of tokens conditioned on by $\pi_\theta$.

\paragraph{Assumptions.}
We make the following assumptions throughout this section.

\begin{assumption}[Deletion-only augmentation]
\label{ass:local-edit}
$\mathcal{E}$ neither inserts new tokens nor reorders existing ones; equivalently, $\tilde\tau$ is a (possibly discontiguous) token subsequence of $\tau$. In particular, $|\tilde\tau|\le T$.
\end{assumption}

\begin{assumption}[Bounded conditional probabilities]
\label{ass:bounded-prob}
There exists $p_{\min}\in(0,1]$ such that, for every prefix $h$ encountered along $\tau$, the conditional next-token distribution under the current policy satisfies $\pi_\theta(\cdot\mid x,h)\in[p_{\min},1]$ for every token in the support of $\tau$ at that step.
\end{assumption}

We additionally assume $\pi_\theta(\cdot\mid x)$ is differentiable in $\theta$ wherever needed, which is automatic for transformer policies. Assumption~\ref{ass:local-edit} matches the operational design of CLORE, since the cleaner's system prompt (Appendix~\ref{app:momo_prompt}) instructs the augmenter to delete spans without paraphrasing or reordering, and the correctness-consistency filter $R(x,\tilde\tau)\ge R(x,\tau)$ in the Method section discards any $\tilde\tau$ whose final answer diverges from $\tau$. Assumption~\ref{ass:bounded-prob} is a mild technical floor that holds throughout training because the underlying transformer assigns strictly positive probability to every vocabulary token; in practice $p_{\min}$ can be enforced by softmax temperature or label smoothing, and we use it only to obtain non-asymptotic bounds.

\paragraph{Joint gradient as a weighted MLE.}

\begin{proposition}[CLORE as a weighted MLE on $(\tau,\tilde\tau)$]
\label{prop:weighted-mle}
The gradient of $\mathcal{L}$ with respect to $\theta$ admits the decomposition
\begin{equation}
\nabla_\theta \mathcal{L}
\;=\; -\,\underbrace{\lambda\beta\,\sigma(-\Delta)}_{\alpha}\,\nabla_\theta\log\pi_\theta(\tau\mid x)
\;-\; \underbrace{\big[A(x,\tilde\tau)-\lambda\beta\,\sigma(-\Delta)\big]}_{A(x,\tilde\tau)\,-\,\alpha}\,\nabla_\theta\log\pi_\theta(\tilde\tau\mid x),
\label{eq:weighted-mle}
\end{equation}
i.e.\ the negative gradient of a weighted log-likelihood on the trajectory pair $(\tau,\tilde\tau)$, with non-negative weight $\alpha = \lambda\beta\sigma(-\Delta)$ on $\tau$ and signed weight $A(x,\tilde\tau)-\alpha$ on $\tilde\tau$.
\end{proposition}

\begin{proof}
The PG term contributes $\nabla_\theta \mathcal{L}_{\mathrm{PG}} = -A(x,\tilde\tau)\,\nabla_\theta\log\pi_\theta(\tilde\tau\mid x)$ directly. For the DPO term, write $f(z)=-\log\sigma(z)$, so that $f'(z) = -\sigma(-z)$. Since $\Delta$ depends on $\theta$ through both log-probabilities, the chain rule gives
\[
\nabla_\theta \mathcal{L}_{\mathrm{DPO}}
\;=\; f'(\Delta)\,\nabla_\theta\Delta
\;=\; -\sigma(-\Delta)\cdot\beta\,\big[\nabla_\theta\log\pi_\theta(\tau\mid x)-\nabla_\theta\log\pi_\theta(\tilde\tau\mid x)\big].
\]
Multiplying by $\lambda$ and adding $\nabla_\theta \mathcal{L}_{\mathrm{PG}}$ collects coefficients on $\nabla_\theta\log\pi_\theta(\tau\mid x)$ and $\nabla_\theta\log\pi_\theta(\tilde\tau\mid x)$, yielding~\eqref{eq:weighted-mle}.
\end{proof}

Equation~\eqref{eq:weighted-mle} makes the role of each component explicit. PG carries the main reinforcement signal on the cleaned trajectory $\tilde\tau$ with weight $A(x,\tilde\tau)$, raising $\pi_\theta(\tilde\tau\mid x)$ as in any standard advantage-weighted update. The DPO term then acts as a regularizer, folding the original on-policy rollout $\tau$ explicitly back into the loss with weight $+\alpha$ on $\nabla_\theta\log\pi_\theta(\tau\mid x)$ and $-\alpha$ on $\nabla_\theta\log\pi_\theta(\tilde\tau\mid x)$, which anchors the policy to its on-policy distribution and prevents PG from drifting too aggressively toward the off-policy cleaned trajectory. The contrastive subtraction $-\alpha$ on $\tilde\tau$ is what mathematically distinguishes the joint objective from a naive mixture of PG with SFT on the cleaner's outputs. Empirically, the value of this regularization shows up in Table~\ref{tab:eval}, where the joint objective beats the pure-PG backbone across all five math test sets. The remaining two results bound and characterize $\alpha$.

\paragraph{Boundedness.}

\begin{lemma}[Bounded DPO logit under local deletion]
\label{lem:bounded-delta}
Under Assumptions~\ref{ass:local-edit}--\ref{ass:bounded-prob},
\[
|\Delta| \;\le\; \beta\,T\,\big|\!\log p_{\min}\big|,
\qquad\text{hence}\qquad
\sigma(-\Delta) \;\in\; \big[\sigma(-\beta T|\!\log p_{\min}|),\,\sigma(\beta T|\!\log p_{\min}|)\big]\subset(0,1).
\]
In particular, the regularization weight in~\eqref{eq:weighted-mle} satisfies $\alpha\le \lambda\beta$.
\end{lemma}

\begin{proof}
By autoregressive factorization, $\log\pi_\theta(\tau\mid x) = \sum_{t=1}^{T}\log\pi_\theta(\tau_t\mid x,\tau_{<t})$, and similarly for $\tilde\tau$ over its $|\tilde\tau|\le T$ tokens. Assumption~\ref{ass:bounded-prob} bounds each conditional in $[p_{\min},1]$, so each summand lies in $[\log p_{\min},0]$. Hence $\log\pi_\theta(\tau\mid x)\in[T\log p_{\min},0]$ and $\log\pi_\theta(\tilde\tau\mid x)\in[|\tilde\tau|\log p_{\min},0]\subseteq[T\log p_{\min},0]$, where the last inclusion uses $|\tilde\tau|\le T$ (Assumption~\ref{ass:local-edit}) together with $\log p_{\min}\le 0$. Subtracting and multiplying by $\beta$ gives $|\Delta|\le \beta T|\!\log p_{\min}|$. The bounds on $\sigma(-\Delta)$ follow from monotonicity of $\sigma$, and $\alpha\le\lambda\beta$ is immediate from $\sigma(-\Delta)\le 1$.
\end{proof}

An unrestricted rewriter that inserts low-probability tokens or reorders content could drive $|\Delta|\to\infty$ and saturate the DPO gradient. Assumption~\ref{ass:local-edit} rules this out and lets CLORE drop the explicit reference policy of standard DPO, with the deletion constraint anchoring $\tilde\tau$ to $\tau$ at the data level and PG anchoring $\pi_\theta$ to $\tilde\tau$ at the parameter level, so DPO acts as a bounded regularizer rather than competing with the main update. Empirically, the linear scaling $\alpha\le\lambda\beta$ matches the DPO-weight sweep in Figure~\ref{fig:ablation}, where increasing $\lambda$ from $0.05$ to $0.2$ keeps training stable rather than diverging.

\paragraph{Self-extinguishing.}

\begin{proposition}[Self-extinguishing DPO term at the no-edit fixed point]
\label{prop:self-extinguish}
If $\tilde\tau = \tau$ (no edits), then $\Delta\equiv 0$, $\mathcal{L}_{\mathrm{DPO}}\equiv\log 2$, and $\nabla_\theta\mathcal{L}_{\mathrm{DPO}}\equiv 0$.
\end{proposition}

\begin{proof}
$\tilde\tau=\tau$ gives $\log\pi_\theta(\tilde\tau\mid x)=\log\pi_\theta(\tau\mid x)$, hence $\Delta=0$, $-\log\sigma(0)=\log 2$, and $\sigma(-\Delta)=\nicefrac{1}{2}$. Substituting into~\eqref{eq:weighted-mle} yields equal coefficients $-\lambda\beta/2$ on $\nabla_\theta\log\pi_\theta(\tau\mid x)$ and $+\lambda\beta/2$ on $\nabla_\theta\log\pi_\theta(\tilde\tau\mid x)$, which cancel because $\tilde\tau=\tau$.
\end{proof}

Proposition~\ref{prop:self-extinguish} shows that CLORE implements an implicit annealing schedule on the DPO regularizer without any hand-tuned schedule. The regularization is loud while $\tilde\tau$ still differs noticeably from $\tau$ and there is substantial content for the cleaner to delete, and it quietly fades once the policy has internalized the deletion behavior so that $\tilde\tau$ is already close to $\tau$ and PG alone keeps reinforcing the cleaned form. This matches the augmentation-analysis trend in Figure~\ref{fig:augmentation_analysis}, where the fraction of removed tokens shrinks and the fraction of retained tokens grows as training progresses.

\section{Experiment implementation details}
\label{app:val_setup}

We monitor model behavior during RL training by periodically running an evaluation rollout on a fixed suite of math test sets. Unless otherwise stated, all length curves and accuracy values in figures plotted against training step come from this evaluation pipeline.

\paragraph{Validation datasets.}
We evaluate on five standard math reasoning benchmarks, summarized in Table~\ref{tab:val_datasets}. For AIME~2025 and AMC~2023 we replicate each unique problem 32 times in the evaluation file, so the reported per-prompt accuracy is equivalent to averaging 32 samples per problem on the 30 (AIME~2025) or 40 (AMC~2023) unique problems.

\begin{table}[h]
\centering
\small
\begin{tabular}{lcc}
\toprule
Dataset & \# prompts & \# unique problems \\
\midrule
MATH500       & 500     & 500 \\
AIME 2025     & 960     & 30  \\
AMC 2023      & 1{,}280 & 40  \\
Minerva Math  & 272     & 272 \\
OlympiadBench & 675     & 675 \\
\bottomrule
\end{tabular}
\caption{Evaluation test sets used during training.}
\label{tab:val_datasets}
\end{table}

\textbf{Sampling configuration.} For each evaluation prompt we draw a single rollout with $\text{temperature}=1.0$, $\text{top-}p=0.7$, and $\text{top-}k$ disabled. Prompts are capped at 1024 tokens. The maximum response length matches the corresponding training configuration: $L_{\max}=8192$ for DeepSeek-R1-Distill-Qwen-7B runs and $L_{\max}=3072$ for Qwen2.5-Math-7B runs. ThinkPrune variants instead follow the curriculum's shorter cap, which is 1536 for Qwen2.5-Math-7B and 4096 or 6144 for DeepSeek-R1.

\textbf{Frequency and batching.} We evaluate every 20 training steps, with an evaluation batch size of 256 for DeepSeek-R1 runs and 512 for Qwen2.5-Math-7B runs. With roughly 540 training steps per run, this yields about 27 evaluation checkpoints per run, which we use for all length-vs-step curves.

\textbf{Scoring.} We extract the last \verb|\boxed{...}| span from each response with a balanced-brace parser, apply standard answer normalization (whitespace, formatting commands, common notational variants), and mark the response correct iff the normalized prediction exactly matches the normalized ground truth.

\textbf{AE score setting.} We provide additional details on the AE score used in the main experiments. 
AE score is designed to evaluate the trade-off between preserving task accuracy and reducing reasoning cost. 
For each method, we compare its accuracy and average output length against a corresponding base model without training as the baseline. 
Let $\mathrm{Acc}_{\mathrm{model}}$ and $\mathrm{Length}_{\mathrm{model}}$ denote the accuracy and average output length of the evaluated model, and let $\mathrm{Acc}_{\mathrm{baseline}}$ and $\mathrm{Length}_{\mathrm{baseline}}$ denote those of the baseline model. We define the relative length reduction and relative accuracy change as
\[
\Delta \mathrm{Length}
=
\frac{
\mathrm{Length}_{\mathrm{baseline}}-\mathrm{Length}_{\mathrm{model}}
}{
\mathrm{Length}_{\mathrm{baseline}}
},
\qquad
\Delta \mathrm{Acc}
=
\frac{
\mathrm{Acc}_{\mathrm{model}}-\mathrm{Acc}_{\mathrm{baseline}}
}{
\mathrm{Acc}_{\mathrm{baseline}}
}.
\]
A positive $\Delta \mathrm{Length}$ indicates that the evaluated model produces shorter responses than the baseline, while a positive $\Delta \mathrm{Acc}$ indicates improved accuracy. The Accuracy--Efficiency Score is then computed as
\[
\mathrm{AE score}
=
\begin{cases}
\alpha \cdot \Delta \mathrm{Length}
+
\beta \cdot |\Delta \mathrm{Acc}|, 
& \text{if } \Delta \mathrm{Acc} \geq 0, \\[4pt]
\alpha \cdot \Delta \mathrm{Length}
-
\gamma \cdot |\Delta \mathrm{Acc}|, 
& \text{if } \Delta \mathrm{Acc} < 0,
\end{cases}
\]
where $\alpha>0$, $\beta>0$, and $\gamma>0$ control the relative importance of length reduction, accuracy improvement, and accuracy degradation, respectively. In our experiments, we set $\alpha=1$, $\beta=3$, and $\gamma=5$. This choice assigns positive credit to output compression and accuracy improvement, while imposing a stronger penalty on accuracy degradation by setting $\gamma>\beta$. Therefore, a higher AE score indicates a better accuracy--efficiency trade-off, whereas a lower AE score indicates that the reduction in reasoning length is insufficient to compensate for accuracy loss.

\textbf{Other Implementation.} We use vLLM~0.11 as the inference engine for both training and evaluation, with \texttt{enforce\_eager=False} and \texttt{enable\_chunked\_prefill=True}, on top of the verl training framework. Prompts follow the model's native chat template (e.g.\ \verb!<|im_start|>!\texttt{user/assistant} tags for Qwen-family models), and we fix the sampling seed to \texttt{1234} so that evaluation rollouts are reproducible across ablation runs.

\section{Illegible reasoning analysis setting}
\label{app:ill_reasoning}
\paragraph{LLM-as-a-judge models.}
To reduce judge-specific bias, we average the result over three strong commercial LLMs as independent judges,
\texttt{gpt-5.5}, \texttt{claude-sonnet-4-6}, and \texttt{gemini-2.5-pro}.
Each judge receives the same evaluation prompt and estimates the percentage of illegible or task-irrelevant content in the reasoning trajectory. The final illegible reasoning rate reported is the average of over all evaluated responses. Lower values indicate that the method produces reasoning traces with fewer uninterpretable, semantically incoherent, or task-irrelevant segments.

\paragraph{LLM-as-a-judge prompt.}
We use the following prompt for each LLM judge:

\begin{quote}
You are evaluating the legibility of a model-generated reasoning trace for a mathematical problem.

Given the problem and the model response, estimate what percentage of the model response consists of illegible or task-irrelevant reasoning content.

Illegible or task-irrelevant reasoning includes:
\begin{itemize}
    \item semantically meaningless or corrupted text;
    \item irrelevant digressions unrelated to solving the problem;
    \item unexplained or nonsensical symbol manipulation;
    \item garbled multilingual text or non-natural-language spans that do not contribute to the solution;
    \item reasoning steps that are impossible to interpret as part of the problem-solving process.
\end{itemize}

Do not judge whether the final answer is correct. A response can have an incorrect final answer but still be legible. A response can also have a correct final answer but contain illegible reasoning.

Estimate the percentage of the response that is illegible or task-irrelevant. Return only a single number between \texttt{0} and \texttt{100}. Do not include any explanation.

Problem:
\{problem\}

Model response:
\{response\}

Illegible percentage:
\end{quote}

\section{Prompt for Augmentation Model}
\label{app:momo_prompt}

The CLORE augmentation model is invoked through a chat-formatted prompt that forbids rewriting and only allows span deletion. We use one-shot prompting, the augmentation model is provided with examples \emph{(question \& reasoning $\rightarrow$ augmented reasoning)} before being asked to augment the target rollout.

\paragraph{System instruction.}
\begin{quote}\small\itshape
You are a strict reasoning filter. Only delete spans from the original reasoning. Do not paraphrase, rewrite, reorder, or add anything. The output must be a subsequence of the input, keeping most of the original reasoning.
\end{quote}

\paragraph{User template.}
Each rollout is wrapped in a single user message of the form
\begin{quote}\small\ttfamily
\{\{instruction\}\}\\[2pt]
Question:\\
\{\{question\}\}\\[2pt]
Reasoning to augment:\\
\{\{rollout\}\}\\[2pt]
augmented reasoning:
\end{quote}
where \verb|{{instruction}}| carries the answer-format instruction described next.

\paragraph{Answer-format constraint.}
The instruction tells the augmentation model what the augmented reasoning must terminate in. For math benchmarks the augmented reasoning must end in a \verb|\boxed{...}| expression; for GSM8K-style splits it must end with \texttt{Answer:~} following the original GSM8K convention.

\paragraph{One-shot example.}
We prepend a fixed (question, original, augmented) triple to anchor the augmentation model on deletion-only behavior. The original reasoning contains a deliberately noisy Python code block (with garbled multilingual tokens and a \texttt{SyntaxError}) and an extraneous verification step; the augmentation model is expected to delete these while preserving the natural-language solution verbatim. The math-format variant is shown below; the GSM8K variant has identical structure but terminates in \texttt{Answer:~10}.

\medskip
\noindent\textsc{Question.}\quad Weng earns \$12 an hour for babysitting. Yesterday, she just did 50 minutes of babysitting. How much did she earn?

\medskip
\noindent\textsc{Original reasoning}:
\begin{quote}\small
Let's break it down step by step. We start by identifying the hourly pay rate \dots Since babysitting pay is given per hour and the time worked is given in minutes, it is natural to convert the hourly rate into a per-minute rate \dots However, to be extra sure that no arithmetic error has slipped in, we can quickly double-check this computation using Python:
\begin{verbatim}
```python
rate = 12/60
minutes = 50
earnings = rate * minutes
print(rate, earnings)
def is_valid so Alb([_##tion_##ing_xQ7z_Pm##_##er_NaN_endoftext_])
0.25 12.5
SyntaxError: closing parenthesis ')' 
does not match opening parenthesis '['
\end{verbatim}
The output above confirms that the conversion and multiplication were done correctly \dots Based on this calculation, Weng earned \$10 \dots \verb|\boxed{10}|
\end{quote}

\noindent\textsc{Augmented reasoning}:
\begin{quote}\small
Let's break it down step by step. We start by identifying the hourly pay rate \dots Since babysitting pay is given per hour and the time worked is given in minutes, it is natural to convert the hourly rate into a per-minute rate before continuing. There are 60 minutes in an hour, so we compute $12/60 = 0.2$ per minute \dots Now, Weng worked for 50 minutes in total. Using the per-minute rate, we multiply $0.2$ by $50$ to obtain her earnings. This gives $0.2 \times 50 = 10$ \dots Based on this calculation, Weng earned \$10 for her babysitting work. \verb|\boxed{10}|
\end{quote}

\paragraph{Sampling settings.}
We serve the augmentaion model with vLLM~0.11 (\texttt{enforce\_eager=False}, \texttt{enable\_chunked\_prefill=True}), sampling at $\text{temperature}=0.7$, $\text{top-}p=0.9$, $n=1$, and \texttt{max\_tokens} set to the length of the rollout being augmented plus a small slack. The same parameters apply to both augmentation model sizes.

\section{Computation Cost}
\label{app:compute_cost}

\paragraph{Hardware setting.}
All experiments run on an internal H200 cluster partition. The project uses an aggregate fleet of 24 H200 (141\,GB) GPUs. Each policy run takes 2 H200s, since the Adam state and activations on a 7B model at $L=8192$ overflow a single 141\,GB H200; CLORE-augmented runs additionally allocate 1 H200 to the external augmentation model served by vLLM. A CLORE-augmented run therefore consumes 3 H200s concurrently while non-CLORE baselines use 2, so the 24-GPU fleet co-schedules up to $\lfloor 24/3 \rfloor = 8$ CLORE ablations or up to 12 baselines in parallel. We use a Singularity container with vLLM~0.11 and FlashAttention-2.8.1.

\paragraph{Per-step phase breakdown.}
We log three phases per training step: \emph{rollout} (vLLM generation), \emph{update} (actor forward, backward, and optimizer), and \emph{augmentation} (sending correct rollouts to the external augmentation model and collecting augmented result). Mean per-step wall-clock for the principal configurations is shown in Table~\ref{tab:compute_per_step}.

\begin{table}[h]
\centering
\small
\setlength{\tabcolsep}{4.5pt}
\begin{tabular}{llccccc}
\toprule
Method & Backbone & GPUs/run & Rollout & Update & Aug. & Total \\
\midrule
GRPO & Qwen2.5-Math-7B & 2 & 72  & 97 & --  & 194 \\
GRPO + CLORE & Qwen2.5-Math-7B & 3 & 55  & 86 & 160 & 320 \\
\midrule
DAPO-overlong & Qwen2.5-Math-7B & 2 & 58  & 82 & --  & 162 \\
DAPO-overlong + CLORE & Qwen2.5-Math-7B & 3 & 43  & 75 & 80  & 214 \\
\midrule
Efficient-RLOO & Qwen2.5-Math-7B & 2 & 56  & 63 & --  & 136 \\
Efficient-RLOO + CLORE & DeepSeek-R1-7B & 3 & 128 & 83 & 234 & 296 \\
\bottomrule
\end{tabular}
\caption{Average per-step wall-clock time. We report rollout, policy-update, and augmentation time in seconds, averaged over approximately 540 training steps. The augmentation stage can partially overlap with the next rollout, so the total step time can be smaller than the sum of individual phases.}
\label{tab:compute_per_step}
\end{table}

The augmentation model adds 80--234\,s per step depending on rollout length, which translates to an additional 33\% (DAPO-overlong) to 65\% (GRPO) of per-step wall-clock on Qwen2.5-Math-7B. Because CLORE simultaneously shortens the policy's rollouts (e.g.\ Qwen GRPO: $1129 \to 618$ tokens), the rollout phase itself becomes faster and partially absorbs the augmentation model overhead.

\paragraph{Per-run GPU-hours.}
A typical 540-step run consumes the budget shown in Table~\ref{tab:compute_per_run}.

\begin{table}[h]
\centering
\small
\begin{tabular}{l c c c}
\toprule
Configuration & wall-clock (h) & GPUs & total per run (H200-h) \\
\midrule
Qwen2.5-Math-7B baseline (no CLORE)  & 20--32 & 2     &  40--64 \\
Qwen2.5-Math-7B + CLORE              & 32--48 & 2 + 1 &  95--145 \\
DeepSeek-R1-Distill-Qwen-7B + CLORE  & 41--72 & 2 + 1 & 120--220 \\
\bottomrule
\end{tabular}
\caption{Per-run compute envelope. CLORE adds the dedicated 1-GPU augmentation model. DeepSeek-R1 runs are roughly $2\times$ longer than Qwen runs because of the larger \texttt{max\_response\_length} (8192 vs.\ 3072).}
\label{tab:compute_per_run}
\end{table}

\paragraph{Total computational cost.}
The whole work consists of approximately 70 finished training runs and consumed roughly 3{,}060 H200-hours in aggregate ($\approx 2{,}630$ for policy training and $\approx 430$ for the augmentation model). The subset of runs whose checkpoints appear in the main paper accounts for roughly half of that ($\sim 1{,}800$ H200-hours). At full 24-GPU utilization the project budget corresponds to about 5--6 calendar days of fully-loaded compute.

\paragraph{FLOPs estimatation.}
We use the standard approximations $C_{\text{rollout}} \approx 2NT$ (forward only) and $C_{\text{update}} \approx 6NT$ (forward, backward, optimizer) with $N \approx 7\times 10^{9}$ for the 7B policy, batch $B=128$ prompts, $n=8$ samples, mean response length $\bar L$, and $S \approx 540$ steps, giving per-run policy FLOPs $C_{\text{policy}} \approx 8 N \bar L B n S$. The Qwen GRPO baseline ($\bar L \approx 1100$) costs about 45 ExaFLOP; Qwen GRPO + CLORE ($\bar L \approx 600$) drops to about 26 ExaFLOP; Qwen DAPO + CLORE ($\bar L \approx 700$) is about 31 ExaFLOP; and DeepSeek-R1 + CLORE ($\bar L \approx 4300$) is about 170 ExaFLOP. The augmentation model contributes $C_{\text{augmentation model}} \approx 2 N_c \bar L_c m S$ with $N_c \approx 4\times 10^{9}$ (Qwen3-4B-Instruct), augmented length $\bar L_c \approx 1500$, and $m \approx 500$ augmented samples per step, giving about 3.2 ExaFLOP per Qwen run, i.e.\ around 10\% of the policy compute on Qwen2.5-Math-7B and around 3\% on DeepSeek-R1. The Qwen3-1.7B augmentation model roughly halves this overhead.

\paragraph{Net FLOPs comparison.}
Despite the additional augmentation model pass and the DPO-Edit term, CLORE+GRPO ends up with lower total FLOPs than the GRPO baseline on Qwen2.5-Math-7B ($26 + 3.2 = 29.2$ vs.\ $45$ ExaFLOP) because the policy generates much shorter rollouts. On DeepSeek-R1 the augmentation model overhead is around 3\% on top of an already-larger policy budget, and CLORE again reduces the policy's rollout length, so total compute is comparable to a non-CLORE baseline of the same step count.

\paragraph{Evaluation cost.}
Post-training inference runs on an H200 cluster, with each (checkpoint, sampling setting, benchmark) tuple submitted as a separate vLLM~0.11 job. A Qwen2.5-Math-7B inference job (1 GPU, $L_{\max}=3072$, five benchmarks totalling 3{,}687 prompts) takes 25--40 minutes per (checkpoint, setting); a DeepSeek-R1 inference job (TP$=$2, $L_{\max}=8192$, same five benchmarks) takes 60--90 minutes. We report three sampling settings per checkpoint ($T{=}1.0,\, p{=}1.0$; $T{=}0$; $T{=}1.0,\, p{=}0.7$). For 14 checkpoints $\times$ 3 settings ($\approx 42$ inference jobs) plus ablation re-runs, the post-training evaluation budget is roughly 30 H200-hours on Qwen and 35 H200-hours on DeepSeek-R1, totalling about 65 H200-hours.

\paragraph{Summary.}
CLORE adds one H200 per active run, so the 24-GPU fleet supports 8 concurrent CLORE ablations. A single run costs roughly 95--145 H200-hours on Qwen2.5-Math-7B and 120--220 H200-hours on DeepSeek-R1, and the full project consumed about 3{,}060 H200-hours of training plus 65 H200-hours of evaluation, equivalent to 5--6 calendar days at full fleet utilization. The augmentation model adds 3--10\% policy FLOPs but the rollout shortening more than compensates, leaving CLORE's net policy FLOPs below the corresponding non-CLORE baseline.

\end{CJK*}

\end{document}